\documentclass[10pt,reqno]{amsart}
\usepackage[utf8]{inputenc}
\usepackage{alias}

\title{Improved Optimistic Algorithms for Logistic Bandits}
\author{Louis Faury $^{1,2,\pmb{*}}$\qquad Marc Abeille $^{1,\pmb{*}}$\qquad Cl\'ement Calauz\`enes $^{1}$\qquad Olivier Fercoq $^{2}$}
\address{$1$ Criteo AI Lab, 32 Rue Blanche, Paris, France.}
\address{$2$ LTCI, T\'el\'ecom Paris, Institut Polytechnique de Paris, Palaiseau, France.}
\address{$\pmb{*}$ Equal contribution.}
\email{Correspondence to \texttt{l.faury@criteo.com}}

\begin{document}
\maketitle

\begin{abstract}
The generalized linear bandit framework has attracted a lot of attention in recent years by extending the well-understood linear setting and allowing to model richer reward structures. It notably covers the logistic model, widely used when rewards are binary. For logistic bandits, the frequentist regret guarantees of existing algorithms are $\bigO{\kappa \sqrt{T}}$, where $\kappa$ is a problem-dependent constant. Unfortunately, $\kappa$ can be arbitrarily large as it scales exponentially with the size of the decision set. This may lead to significantly loose regret bounds and poor empirical performance. In this work, we study the logistic bandit with a focus on the prohibitive dependencies introduced by $\kappa$. We propose a new optimistic algorithm based on a finer examination of the non-linearities of the reward function. We show that it enjoys a $\bigO{\sqrt{T}}$ regret with no dependency in $\kappa$, but for a second order term. Our analysis is based on a new tail-inequality for self-normalized martingales, of independent interest. 

\end{abstract}

\section*{Introduction}

Parametric stochastic bandits is a framework for sequential decision making where the reward distributions associated to each arm are assumed to share a structured relationship through a common unknown parameter. It extends the standard Multi-Armed Bandit framework and allows one to address the exploration-exploitation dilemma in settings with large or infinite action space. Linear Bandits (LBs) are the most famous instance of parametrized bandits, where the value of an arm is given as the inner product between the arm feature vector and the unknown parameter. While the theoretical challenges in LBs are relatively well understood and addressed (see~\citep{dani2008stochastic,rusmevichientong2010linearly,abbasi2011improved,abeille2017linear} and references therein), their practical interest is limited by the linear structure of the reward, which may fail to model real-world problems. As a result, extending LBs to allow for richer reward structures and go beyond linearity has attracted a lot of attention from the bandit community in recent years. To this end, two main approaches have been investigated. Following~\cite{valko2013finite}, the linearity of the reward structure has been relaxed to hold only in a reproducing kernel Hilbert space. Another line of research relies on Generalized Linear Models (GLMs) to encode non-linearity through a link function. We focus in this work on the second approach.

\paragraph{Generalized Linear Bandits.}
The use of generalized linear models for the bandit setting was first studied by \citet{filippi2010parametric}. They introduced GLM-UCB, a generic optimistic algorithm that achieves a $\tilde{O}(d\sqrt{T})$ frequentist regret. In the finite-arm case,~\citet{li2017provably} proposed SupCB-GLM for which they proved a $\tilde{O}(\sqrt{d\log K}\sqrt{T})$ regret bound. Similar regret guarantees were also demonstrated for Thompson Sampling, both in the frequentist \citep{abeille2017linear} and Bayesian \citep{russo2013eluder,russo2014learning,dong2018information} settings. In parallel,~\citet{jun2017scalable} focused on improving the time and memory complexity of Generalized Linear Bandits (GLBs) algorithms while~\cite{dumitrascu2018pg} improved posterior sampling for a Bayesian version of Thompson Sampling in the specific logistic bandits setting.

\paragraph{Limitations.} At a first glance, existing performance guarantees for GLBs seem to coincide with the state-of-the-art regret bounds for LB w.r.t. the dimension $d$ and the horizon $T$. However, a careful examination of the regret bounds shows that they all depend in an \emph{"unpleasant manner on the form of the link function of the GLM, and it seems there may be significant room for improvement"} \citep[\S19.4.5]{lattimore2018bandit}. More in detail, the regrets scale with a multiplicative factor $\kappa$ which characterizes the degree of non-linearity of the link function. As such, for highly non-linear models, $\kappa$ can be prohibitively large, which drastically worsens the regret guarantees as well as the practical performances of the algorithms.

\paragraph{Logistic bandit.}
The magnitude of the constant $\kappa$ is particularly significant for one GLB of crucial practical interest: the logistic bandit. In this case, the link function of the GLB is the sigmoid function, resulting in a highly non-linear reward model. Hence, the associated problem-dependent constant $\kappa$ is large even in typical instances. While this reduces the interest of existing guarantees for the logistic bandit, previous work suggests that there is room for improvement.
In the Bayesian setting and under a slightly more specific logistic bandit instance, \citet{dongon2019} proposed a refined analysis of Thompson Sampling. Their work suggest that in some problem instances, the impact on the regret of the diameter of the decision set (directly linked to $\kappa$) might be reduced.
In the frequentist setting, \citep[\S 4.2]{filippi2010parametric} conjectured that GLM-UCB can be slightly modified in the hope of enjoying an improved regret bound, deflated by a factor $\kappa^{1/2}$. To the best of our knowledge, this is still an open question. 

\paragraph{Contributions.}
In this work, we consider the logistic bandit problem and explicitly study its dependency with respect to $\kappa$. We propose a new non-linear study of optimistic algorithms for the logistic bandit. Our main contributions are : \textbf{1)} we answer positively to the conjecture of \citet{filippi2010parametric} showing that a slightly modified version of GLM-UCB enjoys a $\tilde{O}(d\sqrt{\kappa T})$ frequentist regret (Theorem~\ref{thm:regretofglmimproved}). \textbf{2)} Further, we propose a new algorithm with yet better dependencies in $\kappa$, showing that it can be pushed in a second-order term. This results in a $\tilde{O}(d\sqrt{T}+\kappa)$ regret bound (Theorem~\ref{thm:regretourbandit}). \textbf{3)} A key ingredient of our analysis is a new Bernstein-like inequality for self-normalized martingales, of independent interest (Theorem~\ref{thm:bernsteinselfnormalized}).


\section{Preliminaries}

\paragraph{Notations}
For any vector $x\in\mbb{R}^d$ and any positive definite matrix $\mbold{M}\in\mbb{R}^{d\times d}$, we will note $\mnorm{x}{M}=\sqrt{x^\transp\mbold{M}x}$ the $\ell^2$-norm of $x$ weighted by $\mbold{M}$, and $\lambda_{\text{min}}(\mbold{M})>0$ the smallest eigenvalue of $\mbold M$. For two symmetric matrices $\mbold{A}$ and $\mbold{B}$, $\mbold{A}\succ \mbold{B}$ means that $\mbold{A}-\mbold{B}$ is positive semi-definite. We will denote $\mcal{B}_p(d)=\left\{x\in\mbb{R}^d: \lVert x \rVert_p \leq 1\right\}$ the $d$-dimensional ball of radius 1 under the norm $\ell^p$. For two real-valued functions $f$ and $g$ of a scalar variable $t$, we will use the notation $f_t = \tilde{\mcal{O}}_t(g_t)$ to indicate that $f_t = \mcal{O}(g_t)$ up to logarithmic factor in $t$. For an univariate function $f$ we will denote $\dot{f}$ its derivative.

\subsection{Setting}
\label{sec:setting}
We consider the stochastic contextual bandit problem.
At each round $t$, the agent observes a context and is presented a set of actions $\mcal{X}_t$ (dependent on the context, and potentially infinite). The agent then selects an action $x_t\in\mathcal{X}_t$ and receives a reward $r_{t+1}$. Her decision is based on the information gathered until time $t$, which can be formally encoded in the filtration $\mathcal{F}_t\defeq \left(\mathcal{F}_0,\sigma(\{x_s, r_{s+1}\}_{s=1}^{t-1})\right)$ where $\mathcal{F}_0$ represents any prior knowledge. In this paper, we assume that conditionally on the filtration $\mathcal{F}_t$, the reward $r_{t+1}$ is binary, and is drawn from a Bernoulli distribution with parameter $\mu(x_t^\transp\theta_*)$. The \emph{fixed} but \emph{unknown} parameter $\theta_*$  belongs to $\mbb{R}^d$, and $\mu(x)\defeq (1+\exp(-x))^{-1}$ is the sigmoid function. Formally:
\begin{equation}
\begin{aligned}
    \mathbb{P}\left(r_{t+1} = 1\;\vert\; x_t,\mathcal{F}_t \right) &=  \mu\left(x_t^\transp\theta_*\right)
\end{aligned}
\label{eq:reward_model}
\end{equation}
Let $x_*^t\defeq \argmax_{x\in\mathcal{X}_t}\mu\left(x^\transp\theta_*\right)$ be the optimal arm. When pulling an arm, the agent suffers an instant \emph{pseudo-regret} equal to the difference in expectation between the reward of the optimal arm $x_*^t$ and the reward of the played arm $x_t$. The agent's goal is to minimize the \emph{cumulative} pseudo-regret up to time $T$, defined as:
\begin{align*}
    R_T &\defeq \sum_{t=1}^T \mu\left(\theta_*^\transp x_*^t\right) - \mu\left(\theta_*^\transp x_t\right).
\end{align*}

Following \cite{filippi2010parametric}, we work under the subsequent assumptions on the problem structure, necessary for the study of GLBs\footnote{Assumption~\ref{asm:arm.set} is made for ease of exposition and can be easily relaxed to $\ltwo{x} \leq X$.} .
\begin{ass}[Bandit parameter]
\label{asm:param}
$\theta_* \in \Theta$ where $\Theta$ is a compact subset of $\mbb{R}^d$. Further, $S\defeq \max_{\theta\in\Theta}\ltwo{\theta}$ is known.
\end{ass}
\begin{ass}[Arm set]
\label{asm:arm.set}
Let $\mcal{X}=\bigcup_{t=1}^\infty \mcal{X}_t$. For all $x\in\mcal{X}$, $\ltwo{x}\leq 1$.
\end{ass}
We let $L=M=1/4$ be the upper-bounds on the first and second derivative of the sigmoid function respectively. 
Finally, we formally introduce the parameter $\kappa$ which quantifies the degree of non-linearity of the sigmoid function over the decision set $(\mcal{X},\Theta)$:
\begin{align}
\kappa \defeq \sup_{x\in\mcal{X},\theta\in\Theta} 1/\dot{\mu}(x^\transp\theta).
\label{eq:kappadef}
\end{align}
This key quantity and its impact are discussed in Section~\ref{sec:challenges}.

\subsection{Reminders on optimistic algorithms}
\label{sec:optimistic}
At round $t$, for a given estimator $\theta_t$ of $\theta_*$ and a given exploration bonus $\epsilon_t(x)$, we consider optimistic algorithms that play:
$$
    x_t = \argmax_{x\in\mcal{X}_t} \mu(\theta_t^\transp x) + \epsilon_t(x)
$$
We will denote $\Delta^{\text{pred}}(x,\theta_t)\defeq \left\vert \mu(x^\transp \theta_*)-\mu(x^\transp \theta_t)\right\vert$ the \emph{prediction error} of $\theta_t$ at $x$. It is known that setting the bonus to be an upper-bound on the prediction error naturally gives a control on the regret. Informally:
\begin{align*}
\Delta^{\text{pred}}(x,\theta_t)\leq \epsilon_t(x)
~~\Longrightarrow~~ 
R_T \leq 2\sum_{t=1}^{T}\epsilon_t(x_t).
\end{align*}
This implication is classical and its proof is given in Section~\ref{subsec:regret_decomposition} in the supplementary materials. As usual in bandit problems, tighter predictions bounds on $\Delta^{\text{pred}}(x,\theta_t)$ lead to smaller exploration bonus and therefore better regret guarantees, as long as the sequence of bonus can be shown to cumulate sub-linearly. Reciprocally, using large bonus leads to over-explorative algorithms and consequently large regret.

\subsection{Maximum likelihood estimate}
In the logistic setting, a natural way to compute an estimator for $\theta_*$ given $\mathcal{F}_t$ derives from the maximum-likelihood principle. At round $t$, the regularized log-likelihood (or negative cross-entropy loss) can be written as:
\begin{align*}
    \mcal{L}_t^\lambda(\theta) = \sum_{s=1}^{t-1} \Big[r_{s+1}\log\mu(x_s^\transp\theta)+(1-r_{s+1})\log(1-\mu(x_s^\transp\theta))\Big] -\frac{\lambda}{2}\ltwo{\theta}^2.
\end{align*}
$\mcal{L}_t^{\lambda}$ is a strictly concave function of $\theta$ for $\lambda>0$, and the maximum likelihood estimator is defined as $\hat{\theta}_t\defeq \argmax_{\theta\in\mbb{R}^d} \mcal{L}_t^\lambda(\theta)$. In what follows, for $t\geq 1$ and $\theta\in\mbb{R}^d$ we define $g_t(\theta)$ such as:
\begin{align}
    \nabla_\theta \mcal{L}_t^{\lambda}(\theta) = \sum_{s=1}^{t-1}r_{s+1}x_s - \Big(\underbrace{\sum_{s=1}^{t-1}\mu(x_s^\transp \theta)x_s+\lambda\theta}_{{\textstyle{\defeq g_t(\theta)}}}\Big).
    \label{eq:defgt}
\end{align}
We also introduce the Hessian of the negative log-loss:
\begin{align}
\label{eq:defht}
    \mbold{H}_t(\theta) \defeq \sum_{s=1}^{t-1}\dot{\mu}(x_s^\transp\theta)x_sx_s^\transp + \lambda\mbold{I}_d,
\end{align}
as well as the design-matrix $\mbold{V}_t\defeq \sum_{s=1}^{t-1}x_sx_s^\transp + \kappa\lambda\mbold{I}_d$. 

The negative log-loss $\mcal{L}_t^{\lambda}(\theta)$ is known to be a \emph{generalized self-concordant} function \cite{bach2010self}. For our purpose this boils down to the fact that $\vert \ddot\mu\vert \leq \dot\mu$.

\section{Challenges and contributions}
\label{sec:challenges}

\begin{figure*}[t]
\vspace{1cm}

\centering
   \subcaptionbox*{}{
    \begin{tikzpicture}[scale=0.8]
     \begin{axis}[
        axis x line=center,
        axis y line=center,
        xtick=\empty,
        ytick=\empty,
        scaled ticks=false,
        xmin=-4,
        xmax=4,
        ymin=-0.2,
        ymax=1.1,
        xlabel=,
        ylabel=,
   ]
    \path[name path=axis] (axis cs:-1.2,0) -- (axis cs:1.2,0);
     \node[] at (axis cs:-2.0,0.75) {$\pmb{\boxed{\color{blue}\kappa = 5}}$};
     \draw[dotted, line width= 0.3mm] (axis cs: 1.2,0) -- (axis cs:1.2,0.76);
     \draw[dotted, line width= 0.3mm] (axis cs: -1.2,0) -- (axis cs:-1.2,0.23);
     \addplot[domain=-4:4, black, ultra thick,smooth, name path=f] {1/(e^(-x)+1)};
     \addplot[gray!25, fill opacity=0.5] fill between[of=f and axis, soft clip={domain=-1.2:1.2}];
    \node[circle,fill,inner sep=1pt] at (axis cs:1.2,0) {};
    \node[label=270:{{$\max_{x,\theta}x^T\theta$}},,inner sep=1pt] at (axis cs:1.5,0) {};
    \node[circle,fill,inner sep=1pt] at (axis cs:-1.2,0) {};
    \node[label=270:{{$\min_{x,\theta}x^T\theta$}},inner sep=1pt] at (axis cs:-1.5,0) {};
    \addplot[domain=-1:2.5,blue, densely dashed, line width=2] {0.18*(x-1.2)+0.78};
    
    \end{axis}
    \end{tikzpicture}
    }
    \hspace{25pt}
    \subcaptionbox*{}{
    \begin{tikzpicture}[scale=0.8]
     \begin{axis}[
        axis x line=center,
        axis y line=center,
        xtick=\empty,
        ytick=\empty,
        scaled ticks=false,
        xmin=-4,
        xmax=4,
        ymin=-0.2,
        ymax=1.1,
        xlabel=,
        ylabel=,
   ]
     \node[] at (axis cs:-2.0,0.75) {$\boxed{\color{blue}\begin{aligned} \pmb{\kappa}&\pmb{=1000} \\ &\pmb{\sim \exp(z)} \end{aligned}}$};
     \draw[dotted, line width= 0.3mm] (axis cs: 2.0,0) -- (axis cs:2.0,0.997);
      \node[label=270:{{${\max_{x,\theta}x^T\theta}\color{blue}\pmb{=z}$}},circle,fill,inner sep=1pt] at (axis cs:2.0,0) {};
      \node[label=270:{{$\min_{x,\theta}x^T\theta$}},circle,fill,inner sep=1pt] at (axis cs:-2.0,0) {};
       \path[name path=axis] (axis cs:-2.0,0) -- (axis cs:2.0,0);
    \addplot[domain=-4:4, black, ultra thick,smooth, name path=f] {1/(e^(-3*x)+1)};
    \addplot[gray!25, fill opacity=0.5] fill between[of=f and axis, soft clip={domain=-2:2}];
    \addplot[domain=-0.5:4,blue, densely dashed, line width=2] {0.015*(x-2)+1};
    \end{axis}
    \end{tikzpicture}
    }
    \caption{Visualization of the reward signal for different arm-sets and parameter-sets. Left: $\kappa$ is small as the agent mostly plays in the linear part of the sigmoid, a case of little practical interest. Right: $\kappa$ is significantly larger as the agent plays on a larger  spectrum of the sigmoid. This case is more realistic as there exists both actions of very high and very low value.}
    \label{fig:sigmoidcmu}
\end{figure*}
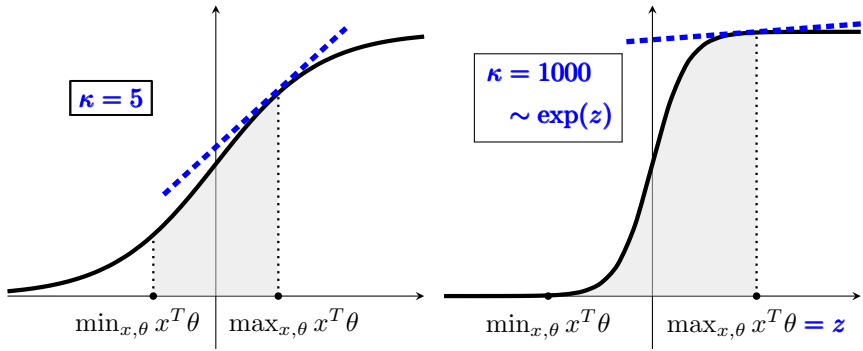

\paragraph{On the scaling of $\pmb{\kappa}$.} First, we 
stress 
the problematic scaling of $\kappa$ (defined in Equation~\eqref{eq:kappadef}) with respect to the size of the decision set $\mcal{X}\times\Theta$. As illustrated in Figure~\ref{fig:sigmoidcmu}, the dependency is exponential and hence prohibitive. From the definition of $\kappa$ and the definition of the sigmoid function, one can easily see that:
\begin{equation}
    \kappa \geq \exp\left(\max_{x\in\mcal{X}}\vert x^\transp\theta_*\vert\right).
    \label{eq:kappalogodds}
\end{equation}
The quantity $x^\transp \theta_*$ is directly linked to the probability of receiving a reward when playing $x$. As a result, this lower bound stresses that $\kappa$ will be exponentially large as soon as there exists bad (resp. good) arms $x$ associated with a low (resp. high) probability of receiving a reward. This is unfortunately the case of most logistic bandit applications. For instance, it stands as the standard for click predictions, since the probability of observing a click is usually low (and hence $\kappa$ is large). Typically, in this setting, $\mbb{P}(\text{click})=10^{-3}$ and therefore $\kappa\sim 10^3$. As all existing algorithms display a linear dependency with $\kappa$ (see Table~\ref{tab:regret_comparison}), this narrows down the class of problem they can efficiently address. On the theoretical side, this indicates that the current analyses fail to handle the regime where the reward function is significantly non-linear, which was the primary purpose of extending LB to GLB. Note that \eqref{eq:kappalogodds} is only a lower-bound on $\kappa$. In some settings $\kappa$ can be even larger: for instance when $\mcal{X}=\mcal{B}_2(d)$, we have $\kappa \geq  \exp(S)$. Even for reasonable values of $S$, this has a disastrous impact on the regret bounds.

\begin{table*}[t]
    \centering
    \begin{tabular}{|c||c|c|}
    \hline
         \textbf{Algorithm} & \textbf{Regret Upper Bound} & \textbf{Note}   \\
         \hline
         \begin{tabular}{c} GLM-UCB\\ \cite{filippi2010parametric}\end{tabular} &  $\mcal{O}\left({\color{black}\pmb{\kappa}}\cdot d \cdot T^{1/2}\cdot \log(T)^{3/2}\right)$& GLM\\ 
         \hline
         \begin{tabular}{c} Thompson Sampling \\ \cite{abeille2017linear}\end{tabular}& $\mcal{O}\left({\color{black}\pmb{\kappa}}\cdot d^{3/2} \cdot T^{1/2} \log(T)\right)$ & GLM\\
         \hline
         \begin{tabular}{c} SupCB-GLM\footnotemark \\ \cite{li2017provably}\end{tabular}& $\mcal{O}\left({\color{black}\pmb{\kappa}}\cdot (d\log K)^{1/2} \cdot T^{1/2} \log(T)\right)$ & GLM, $K$ actions \\
         \hline
         \begin{tabular}{c}\glmimproved \\ \textbf{(this paper)} \end{tabular} & $\mcal{O}\left({\color{black}\pmb{\kappa^{1/2}}}\cdot d \cdot T^{1/2} \log(T)\right)$& Logistic model\\
         \hline
         \begin{tabular}{c}\ourbandit\\ \textbf{(this paper)}\end{tabular} & $\mcal{O}\left( d \cdot T^{1/2} \log(T)+{\color{black}\pmb{\kappa}}\cdot d^2\cdot \log(T)^2\right)$& Logistic model\\
         \hline
    \end{tabular}
    \caption{Comparison of frequentist regret guarantees for the logistic bandit with respect to $\kappa$, $d$ and $T$. $\kappa$ is problem-dependent, and can be prohibitively large even for reasonable problem instances.}
    \label{tab:regret_comparison}
\end{table*}

\paragraph{Uniform vs local control over $\pmb{\dot\mu}$.}
The presence of $\kappa$ in existing regret bounds is inherited from the \emph{learning} difficulties that arise from the logistic regression. Namely, when $\theta_*$ is large, repeatedly playing actions that are closely aligned with $\theta_*$ (a region where $\dot\mu$ is close to 0) will almost always lead to the same reward. This makes the estimation of $\theta_*$ in this direction \emph{hard}. However, this should not impact the regret, as in this region the reward function is \emph{flat}. Previous analyses ignore this fact, as they don't study the reward function \emph{locally} but globally. More precisely, they use both uniform upper ($L$) and lower bounds ($\kappa^{-1}$) for the derivative of the sigmoid $\dot\mu$. Because they are not attained at the same point, at least one of them is loose.
Alleviating the dependency in $\kappa$ thus calls for an analysis and for algorithms that better handle the non-linearity of the sigmoid, switching from a uniform to a local analysis. 
As mentioned in Section~\ref{sec:optimistic}, a thorough control on the \emph{prediction error} $\Delta^{\text{pred}}$ is key to a tight design of optimistic algorithm. The challenge therefore resides in finely handling the locality when controlling the prediction error.

\footnotetext{\citet{li2017provably} uses a definition for $\kappa$ which slightly differs from ours. However, it exhibits the same scaling in $\max\vert x^\transp \theta_*\vert$. We keep this notation to ease discussions.}

\paragraph{On \citet{filippi2010parametric}'s conjecture.} 
In their seminal work, \cite{filippi2010parametric} provided a prediction bound scaling as $\kappa$, directly impacting the size of the bonus. They however hint, by using an asymptotic argument, that this dependency could be reduced to a $\sqrt{\kappa}$. This suggest that a first limitation resides in their concentration tools. 
To this end, we introduce a novel Bernstein-like self-normalized martingale tail-inequality (Theorem~\ref{thm:bernsteinselfnormalized}) of potential independent interest. Coupled with a generalized self-concordant analysis, we give a formal proof of Filippi's asymptotic argument in the finite-time, adaptive-design case (Lemma~\ref{lemma:lemmapredictionbound1}). We leverage this refined prediction bound to introduce \glmimproved. We show that it suffers at most a regret in $\bigO{d\sqrt{\kappa T}}$ (Theorem~\ref{thm:regretofglmimproved}), improving previous guarantees by $\sqrt{\kappa}$. Our novel Bernstein inequality, together with the generalized self-concordance property of the log-loss are key ingredients of local analysis, which allows to compare the derivatives of the sigmoid function at two different points without using $L$ and $\kappa^{-1}$.

\paragraph{Dropping the $\pmb{\kappa}$ dependency.} 
Further challenge is to get rid of the remaining $\sqrt{\kappa}$ factor from the regret. This in turns requires to eliminate it from the bonus of the algorithm. We show that this can be done by pushing $\kappa$ to a second order term in the prediction bound (Lemma~\ref{lemma:predictionboundtwo}). Coupled with careful algorithmic design, this yields \ourbandit, for which we show a $\bigO{d\sqrt{T}+\kappa\log T}$ regret bound (Theorem~\ref{thm:regretourbandit}), where the dependency in $\kappa$ is removed from the leading term.

\paragraph{Outline of the following sections.} Section~\ref{sec:prediction} focuses on exhibiting improved upper-bound on prediction errors. We describe our algorithms and their regret bound in Section~\ref{sec:algos}. Finally, we discuss our results and their implications in Section~\ref{sec:discussion}.




\section{Improved prediction guarantees}
\label{sec:prediction}

This section focuses on the first challenge of the logistic bandit analysis, and aims to provide tighter prediction bounds for the logistic model. Bounding the prediction error relies on building tight confidence sets for $\theta_*$, and our first contribution is to provide more adapted concentration tools to this end. Our new tail-inequality for self-normalized martingale allows to construct such confidence sets with better dependencies with respect to $\kappa$. 

\subsection{New  tail-inequality for self-normalized martingales}

We present here a new, Bernstein-like tail inequality for self-normalized vectorial martingales. This inequality extends known results on self-normalized martingales \citep{de2004self, abbasi2011improved}. Compared to the concentration inequality from Theorem 1 of \cite{abbasi2011improved}, its main novelty resides in considering martingale increments that satisfy a Bernstein-like condition instead of a sub-Gaussian condition. This allows to derive tail-inequalities for martingales "re-normalized" by their quadratic variation.

\begin{restatable}[]{thm}{bernsteinselfnormalized}
\label{thm:bernsteinselfnormalized}
Let $\{\mcal{F}_t\}_{t=1}^\infty$ be a filtration. Let $\{x_t\}_{t=1}^{\infty}$ be a stochastic process in $\mcal{B}_2(d)$ such that $x_t$ is $\mcal{F}_{t}$ measurable. Let $\{\varepsilon_{t}\}_{t=2}^\infty$ be a martingale difference sequence such that $\varepsilon_{t+1}$ is $\mcal{F}_{t+1}$ measurable. Furthermore, assume that conditionally on $\mcal{F}_t$ we have $ \vert \varepsilon_{t+1}\vert \leq 1$ almost surely, and note $\sigma_t^2 \defeq \mbb{E}\left[\varepsilon_{t+1}^2\vert \mcal{F}_t \right]$. Let $\lambda>0$ and for any $t\geq 1$ define:
\begin{align*}
\mbold{H}_t\defeq\sum_{s=1}^{t-1}\sigma_s^2 x_sx_s^T + \lambda\mbold{I}_d, \qquad S_t\defeq \sum_{s=1}^{t-1} \varepsilon_{s+1}x_s.
\end{align*}
Then for any $\delta\in(0,1]$:
\begin{align*}
\mbb{P}\Bigg(\exists t\geq 1, \, \left\lVert S_t\right\rVert_{\mbold{H}_t^{-1}} \!\geq\! \frac{\sqrt{\lambda}}{2}\!+\!\frac{2}{\sqrt{\lambda}}\log\!\left(\frac{\det\left(\mbold{H_t}\right)^{\frac{1}{2}}\!\lambda^{-\frac{d}{2}}}{\delta}\right)+\frac{2}{\sqrt{\lambda}}d\log(2)\Bigg)\leq \delta.
\end{align*}
\end{restatable}
\begin{proof}
The proof is deferred to Section~\ref{sec:proof_bernsteinselfnormalized} in the supplementary materials. It follows the steps of the pseudo-maximization principle introduced in \cite{de2004self}, used by \cite{abbasi2011improved} for the linear bandit and thoroughly detailed in Chapter 20 of \cite{lattimore2018bandit}. The main difference in our analysis comes from the fact that we consider another super-martingale, which adds complexity to the analysis.
\end{proof}

\paragraph{Comparison to prior work} The closest inequality of this type was derived by \citet{abbasi2011improved} to be used for the linear bandit setting. Namely, introducing $\omega \defeq \inf_s \sigma_s^2$, it can be extracted from their Theorem 1 that that with probability at least $1-\delta$ for all $t\geq 1$:
\begin{align}
\label{eq:concentrationyasin}
    \left\lVert S_t\right\rVert_{\mbold{V}_t^{-1}} \leq \sqrt{2d\log\left(1+\frac{\omega t}{\lambda d}\right)},
\end{align}
where $\mbold{V}_t=\sum_{s=1}^{t-1}x_sx_s^T + (\lambda/\omega)\mbold{I}_d$.
Note that this result can be used to derive another high-probability bound on $\left\lVert S_t\right\rVert_{\mbold{H}_t^{-1}}$. Indeed notice that $\mbold{H}_t\succeq \omega \mbold{V}_t$, which yields that with probability at least $1-\delta$:
\begin{align}
    \left\lVert S_t\right\rVert_{\mbold{H}_t^{-1}} \leq \frac{1}{\sqrt{\omega}}\sqrt{2d\log\left(1+\frac{\omega t}{\lambda d}\right)}.
    \label{eq:csababoundonSt}
\end{align}
In contrast the bound given by Theorem~\ref{thm:bernsteinselfnormalized} gives that with high-probability $\left\lVert S_t\right\rVert_{\mbold{H}_t^{-1}} = \mcal{O}\left(d\log(t)\right)$ which is independent of $\omega$. This saves up the multiplicative factor $1/\sqrt{\omega}$, which is potentially very large if some $\varepsilon_s$ have small conditional variance. However, it is lagging by a $\sqrt{d\log(t)}$ factor behind the bound provided in \eqref{eq:csababoundonSt}. This issue can be fixed by simply adjusting the regularization parameter. More precisely, for a given horizon $T$, Theorem~\ref{thm:bernsteinselfnormalized} ensure that choosing a regularization parameter $\lambda= d\log(T)$ yields that on a high-probability event, for all $t\leq T$:
\begin{align*}
    \left\lVert S_t\right\rVert_{\mbold{H}_t^{-1}} = \mcal{O}\left(\sqrt{d\log(T)}\right).
\end{align*}
In this case, our inequality is a \emph{strict} improvement over previous ones, which involved the scalar $\omega$.

\subsection{A new confidence set}
\label{sec:newconfidenceset}
\begin{figure*}
\centering
\begin{subfigure}{0.49\linewidth}
    \centering
    \includegraphics[width=\linewidth]{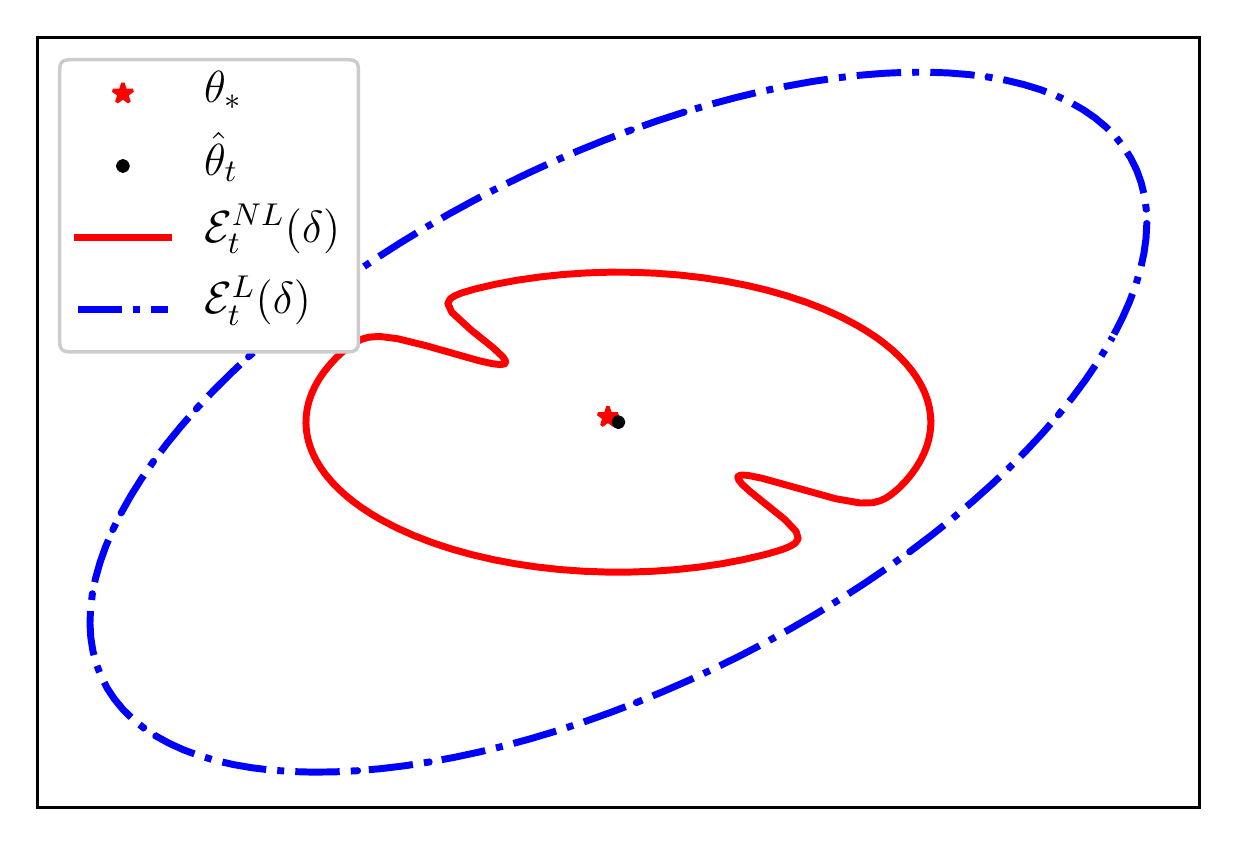}
    \caption{$\kappa=50$, $\delta=0.05$}
    \end{subfigure}
    \begin{subfigure}{0.49\linewidth}
    \includegraphics[width=\linewidth]{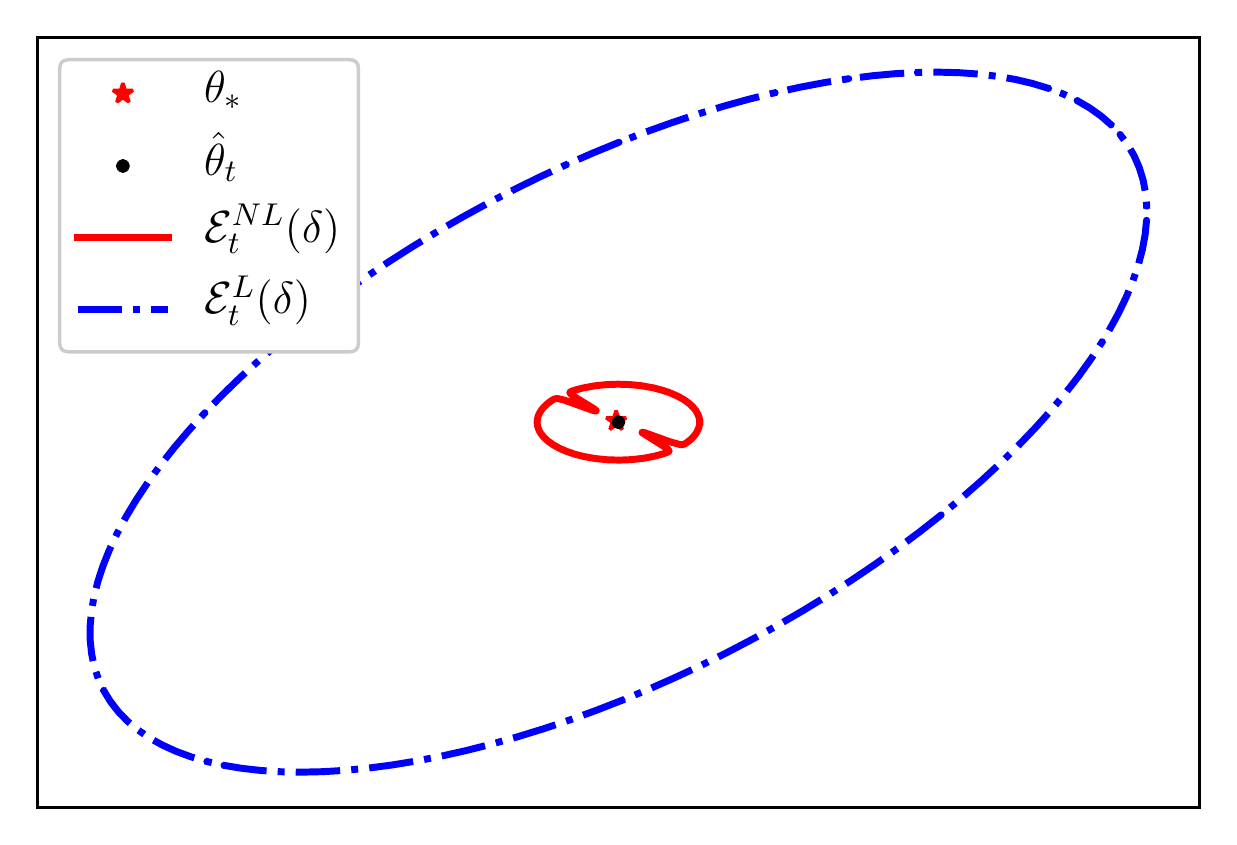}
    \caption{$\kappa=500$, $\delta=0.05$}
    \end{subfigure}
    \caption{Visualization of $\mcal{E}_t^{\text{L}}(\delta)$ and $\mcal{E}_t^{\text{NL}}(\delta)$ for different values of $\kappa$. On both figures, a direction is over-sampled to highlight the non-linear nature of $\mcal{E}_t^{\text{NL}}(\delta)$. As $\kappa$ grows, the difference in diameter between $\mcal{E}_t^{\text{L}}(\delta)$ and $\mcal{E}_t^{\text{NL}}(\delta)$ increases.}
    \label{fig:confidence_sets}
\end{figure*}

We now use our new concentration inequality (Theorem~\ref{thm:bernsteinselfnormalized}) to derive a confidence set on $\theta_*$ that in turns will lead us to upper bounds on the prediction error. We introduce:
\begin{align*}
    \mcal{C}_t(\delta) \!\defeq\! \left\{ \theta\!\in\!\Theta,\, \left\lVert g_t(\theta)-g_t(\hat{\theta}_t)\right\rVert_{\mbold{H}_t^{-1}(\theta)}\leq \gamma_t(\delta)\right\},
\end{align*}
with $g_t$ defined in \eqref{eq:defgt}, $\mbold{H}_t$ in \eqref{eq:defht}, and where
\begin{align*}
    \gamma_t(\delta)\!\defeq \!\sqrt{\lambda}(S\!+\!\frac{1}{2})\!+\! \frac{2}{\sqrt{\lambda}}\!\log\left(\!\frac{2^d}{\delta}\!\left(1\!+\!\frac{Lt}{d\lambda}\right)^\frac{d}{2}\right).
\end{align*}

A straight-forward application of Theorem~\ref{thm:bernsteinselfnormalized} proves that the sets  $\mcal{C}_t(\delta)$ are confidence sets for $\theta_*$.
\begin{restatable}[]{lemma}{lemmaconfidenceset}
\label{lemma:confidenceset}
    Let $\delta\in(0,1]$ and
    \begin{align*}
        E_\delta \defeq \{\forall t\geq 1, \, \theta_*\in\mcal{C}_t(\delta)\}.
    \end{align*}
    Then $\mbb{P}\left(E_\delta\right)\geq 1-\delta$.
\end{restatable}
\begin{sketchofproof}
We show $\left\lVert g_t(\theta_*\!)-g_t(\hat{\theta}_t\!)\right\rVert_{\mbold{H}_t^{-1}(\theta_*\!)}\!\!\leq\! \gamma_t(\delta)$ with probability at least $1-\delta$. As $\nabla\mcal{L}_t^\lambda(\hat{\theta}_t) = 0$, we have
\begin{align*}
    g_t(\hat{\theta}_t)-g_t(\theta_*) = \sum_{s=1}^{t-1}\underbrace{r_{s+1} - \mu(\theta_*^\transp x_s)}_{\defeq \varepsilon_{s+1}}x_s - \lambda \theta_*\,.
\end{align*}
This equality is obtained by using the characterization of $\hat{\theta}_t$ given by the log-loss. By \eqref{eq:reward_model}, $\{\varepsilon_{s+1}\}_{s=1}^\infty$ are centered Bernoulli variables with parameter $\mu(x_s^\transp \theta_*)$, and variance $\sigma_s^2 = \mu(x_s^\transp \theta_*)(1-\mu(x_s^\transp \theta_*))=\dot{\mu}(x_s^T\theta_*)$. Theorem~\ref{thm:bernsteinselfnormalized} leads to the claimed result up to some simple upper-bounding. The formal proof is deferred to Section~\ref{sec:proofconcentration} in the supplementary materials.
\end{sketchofproof}

\paragraph{Illustration of confidence sets.}
We provide here some intuition on how this confidence set helps us improve the prediction error upper-bound. To do so, and for the ease of exposition, we will consider for the remaining of this subsection the case when $\hat{\theta}_t\in\Theta$. We back our intuition on a slightly degraded but more comprehensible version of the upper-bound on the prediction error:
$$
\Delta^{\rm{pred}}(x,\theta) \leq L \lVert x\rVert_{{\bf H}_t^{-1}(\theta)} \lVert \theta - \theta_*\rVert_{{\bf H}_t(\theta)}.
$$
The regret guarantees of our algorithms can still be recovered from this cruder upper-bound, up to some multiplicative constants (for the sake of completeness, technical details are deferred to Section~\ref{sec:proofconfidenceset} in the appendix). The natural counterpart of $\mcal{C}_t$  that allows for controlling the second part of this decomposition is a marginally inflated confidence set,
\begin{align*}
    \mcal{E}_t^{\text{NL}}(\delta) &\!\defeq\! \left\{\!\theta\in\Theta,\left\lVert \theta - \hat{\theta}_t\right\rVert_{\mbold{H}_t(\theta)}\leq  (1+2S) \gamma_t(\delta)\right\}\,.
\end{align*}
It is important to notice (see Figure \ref{fig:confidence_sets}) that $\mcal{E}_t^{\text{NL}}(\delta)$ effectively handles the local curvature of the sigmoid function, as the metric $\mbold{H}_t(\theta)$ is \emph{local} and depends on $\theta$. This results in a confidence set that is not an ellipsoid, and that does not penalize all estimators in the same ways.

Using the same tools as for GLM-UCB, such as the concentration result reminded in \eqref{eq:concentrationyasin}, a similar reasoning leads to the confidence set
\begin{align*}
    \mcal{E}_t^\text{L}(\delta)\defeq  \left\{\theta\in\Theta, \,  \left \lVert \theta - \hat{\theta}_t\right\rVert_{\mbold{V}_t} \leq \kappa \beta_t(\delta)\right\}\,,
\end{align*}
where $\beta_t(\delta)$ is a slowly increasing function of $t$ with similar scaling as $\gamma_t$. Using global bounds on $\dot\mu$ leads to the appearance of $\kappa$ in $\mcal{E}_t^{\rm L}(\delta)$, illustrated by the large difference of diameter between the blue and red sets in Figure~\ref{fig:confidence_sets}. This highlights the fact that the local metric $\mbold{H}_t(\theta)$ is much better-suited than $\mbold{V}_t$ to measure distances between parameters. The intuition laid out in this section underlies the formal improvements on the prediction error bounds we provide in the following.

\subsection{Prediction error bounds}
\label{sec:proofpredictionbounds}
We are now ready to derive our new prediction guarantees, inherited from Theorem~\ref{thm:bernsteinselfnormalized}. 

We give a first prediction bound obtained by \emph{degrading} the local information carried by estimators in $\mcal{C}_t(\delta)$. This guarantee is conditioned on the good event $E_\delta$ (introduced in Lemma \ref{lemma:confidenceset}), which occurs with probability at least $1-\delta$.

\begin{restatable}[]{lemma}{lemmapredictionboundone}
\label{lemma:lemmapredictionbound1}
On the event $E_\delta$, for all $t\geq 1$, any $\theta\in\mcal{C}_t(\delta)$ and $x\in\mcal{X}$:
\begin{align*}
    \Delta^{{\rm pred}}(x,\theta) &\leq L\sqrt{4+8S}\sqrt{\kappa}\gamma_t(\delta) \left\lVert x\right\rVert_{\mbold{V}_t^{-1}}.
\end{align*}
\end{restatable}
In term of scaling with $\kappa$, note that Lemma~\ref{lemma:lemmapredictionbound1} improves the prediction bounds of \citet{filippi2010parametric} by a $\sqrt{\kappa}$. It therefore matches their asymptotic argument, providing its first rigorous proof in finite-time and for the adaptive-design case. The proof is deferred to Section~\ref{sec:prooflemmapredictionbound1} in the supplementary materials.

A more careful treatment of $\mcal{C}_t(\delta)$ naturally leads to better prediction guarantees, laying the foundations to build \ourbandit{}. This is detailed by the following Lemma.

\begin{restatable}[]{lemma}{lemmapredictiontwo}
\label{lemma:predictionboundtwo}
On the event $E_\delta$, for all $t\geq 1$, any $\theta\in\mcal{C}_t(\delta)$ and any $x\in\mcal{X}$:
\begin{align*}
    \Delta^{{\rm pred}}(x,\theta) \leq (2+4S)\dot{\mu}(x^\transp\theta)\left\lVert x\right\rVert_{\mbold{H}_t^{-1}(\theta)}\gamma_t(\delta) + (4+8S)M\kappa\gamma_t^2(\delta)\left\lVert x\right\rVert_{\mbold{V}_t^{-1}}^2.
\end{align*}
\end{restatable}
The proof is deferred to Section~\ref{sec:prooflemmapredictionbound2}.
The strength of this result is that it displays a first-order term that contains only \emph{local} information about the region of the sigmoid function at hand, through the quantities $\dot{\mu}(x^T\theta)$ and $\left\lVert x\right\rVert_{\mbold{H}_t^{-1}(\theta)}$. Global information (measured through $M$ and $\kappa$) are pushed into a second order term that vanishes quickly. Finally, we anticipate on the fact that the decomposition displayed in Lemma~\ref{lemma:predictionboundtwo} is not innocent. In what follows, we will show that both terms cumulate at different rates, the term involving $\kappa$ becoming an explicit second order term. However, this will require a careful choice of $\theta\in\mcal{C}_t$, as the bound on $\Delta^{\rm pred}$ now depends now on $\theta$ (and therefore so will the associated exploration bonus).

\section{Algorithms and regret bounds}
\label{sec:algos}
\subsection{\glmimproved}
We introduce an algorithm leveraging Lemma~\ref{lemma:lemmapredictionbound1}, henceforth matching the heuristic regret bound conjectured in \citep[\S4.2]{filippi2010parametric}. We introduce the feasible estimator:
\begin{align}
    \theta^{(1)}_t = \argmin_{\theta\in\Theta} \left\lVert g_t(\theta)-g_t(\hat{\theta}_t)\right\rVert_{\mbold{H}_t^{-1}(\theta)}\,.
    \label{eq:thetatildedef}
\end{align}
This projection step ensures us that $\theta^{(1)}_t\in\mcal{C}_t(\delta)$ on the high-probability event $E_\delta$. Further, we define the bonus:
\begin{align*}
     \epsilon_{t,1}(x) =  L\sqrt{4+8S}\sqrt{\kappa}\gamma_t(\delta) \left\lVert x\right\rVert_{\mbold{V}_t^{-1}}
\end{align*}
We define \glmimproved{} as the optimistic algorithm instantiated with $(\theta^{(1)}_t,\epsilon_{t,1}(x))$, detailed in Algorithm~\ref{algo:glmimproved}. Its regret guarantees are provided in Theorem~\ref{thm:regretofglmimproved}, and improves previous results by $\sqrt{\kappa}$.

\begin{algorithm}[tb]
   \caption{\glmimproved}
   \label{algo:glmimproved}
\begin{algorithmic}
   \STATE {\bfseries Input:} regularization parameter $\lambda$
   \FOR{$t\geq 1$}
   \STATE Compute $\theta^{(1)}_t$ (Equation~\eqref{eq:thetatildedef})
   \STATE Observe the contexts-action feature set $\mcal{X}_t$.
    \STATE Play $x_t =  \argmax_{x\in\mcal{X}_t} \mu(x^\transp\theta^{(1)}_t) + \epsilon_{t,1}(x)$
    \STATE Observe rewards $r_{t+1}$.\\
   \ENDFOR
\end{algorithmic}
\end{algorithm}

\begin{restatable}[Regret of \glmimproved]{thm}{regretofglmimproved}
\label{thm:regretofglmimproved}
    With probability at least $1-\delta$:
    \begin{align*}
        R_T^{(1)}\!\leq\! C_1 L \sqrt{\kappa}\gamma_{T}(\delta)\sqrt{T}
    \end{align*}
with $C_1 = \sqrt{32d(1+2S)\max(1,1/(\kappa\lambda))\log\left(\!1\!+\!\frac{ T}{\kappa \lambda d}\!\right)}$. Furthermore, if $\lambda = d\log(T)$ then:
\begin{align*}
    R_T^{(1)} = \mcal{O}\left(\sqrt{\kappa}\cdot d \cdot \sqrt{T}\log(T)\right).
\end{align*}
\end{restatable}

\begin{sketchofproof}
Note that by Lemma~\ref{lemma:lemmapredictionbound1} the bonus $\epsilon_{t,1}(x)$ upper-bounds $\Delta^\text{pred}(x,\theta^{(1)}_t)$ on a high-probability event. This ensures that $R_T^{(1)}\leq 2\sum_{t=1}^\transp \epsilon_{t,1}(x_t)$ with high-probability. A straight-forward application of the Elliptical Lemma (see e.g. \citep{abbasi2011improved}, stated in Appendix \ref{sec:proof_useful_lemmas}) ensures that the bonus cumulates sub-linearly and leads to the regret bound. The formal proof is deferred to Section~\ref{sec:proofregretofglmimproved} in the supplementary material.
\end{sketchofproof}

\begin{rem*}
The projection step presented in Equation~\eqref{eq:thetatildedef} is very similar to the one employed in \cite{filippi2010parametric}, to the difference that we use the metric $\mbold{H}_t(\theta)$ instead of $\mbold{V}_t$. While both lead to complex optimization programs (i.e non-convex), neither needs to be carried out when $\hat{\theta}_t\in\Theta$, which can be easily checked online and happens most frequently in practice.
\end{rem*}

\subsection{\ourbandit}
To get rid of the last dependency in $\sqrt{\kappa}$ and improve \glmimproved{}, we use the improved prediction bound provided in Lemma~\ref{lemma:predictionboundtwo}. Namely, we define the bonus:
\begin{align*}
    \epsilon_{t,2}(x,\theta) = (2+4S)\dot{\mu}(x^\transp\theta)\left\lVert x\right\rVert_{\mbold{H}_t^{-1}(\theta)}\gamma_t(\delta) + (4+8S)M\kappa\gamma_t^2(\delta)\left\lVert x\right\rVert_{\mbold{V}_t^{-1}}^2
\end{align*}
However, as this bonus now depends on the chosen estimate $\theta$, existing results (such as the Elliptical Lemma) do not guarantees that it sums sub-linearly. To obtain this property, we need to restrain $\mcal{C}_t(\delta)$ to a set of admissible estimates that, 
intuitively, make the most of the past information already gathered.
Formally, we define the \emph{best-case log-odds} at round $s$ by $\ell_s\defeq \max_{\theta'\in\mcal{C}_s(\delta)\cap\Theta} \vert x_s^\transp\theta'\vert$, and the set of admissible log-odds at time $t$ as:
\begin{align*}
    \mcal{W}_t = \left\{\theta\in\Theta \text{ s.t } \vert \theta^\transp x_s\vert  \leq \ell_s, \, \forall s\leq t-1\right\}.
\end{align*}

Note that $\mcal{W}_t$ is made up of $\max(\vert \mcal{X}\vert,t\!-\!1)$ convex constraints, and is trivially not empty when $0_d\in\Theta$. Thanks to this new feasible set, we now define the estimator:
\begin{align}
    \theta^{(2)}_t \defeq \argmin_{\theta\in \mcal{W}_t} \left\lVert g_t(\theta)-g_t(\hat{\theta}_t)\right\rVert_{\mbold{H}_t^{-1}(\theta)}
    \label{eq:definionthetabar}
\end{align}

We define \ourbandit{} as the optimistic bandit instantiated with $(\theta^{(2)}_t,\epsilon_{t,2}(x,\theta^{(2)}_t))$ and detailed in Algorithm~\ref{algo:ourbandit}. We state its regret upper-bound in Theorem~\ref{thm:regretourbandit}. This result shows that the dominating term (in $\sqrt{T})$ of the regret is \emph{independent} of $\kappa$. A dependency still exists, but for a \emph{second-order term} which grows only as $\log(T)^2$.

\begin{restatable}[Regret of \ourbandit]{thm}{regretourbandit}
\label{thm:regretourbandit}
With probability at least $1-\delta$:
\begin{align*}
    R_T^{(2)} \leq C_2\gamma_{T}(\delta)\sqrt{T}
    + C_3\gamma_{T}^2(\delta)\kappa
\end{align*}
with 
\begin{align*}
C_2&=(4+8S)\sqrt{2dL\max(1,L/\lambda)\log\left(1+\frac{LT}{d\lambda}\right)}\\
C_3 &= Md\max(1,1/(\kappa\lambda))\log\left(1\!+\!\frac{T}{\kappa d\lambda}\!\right)(8+16S)(2+2\sqrt{1+2S})
\end{align*}
Furthermore if $\lambda = d\log(T)$ then:
\begin{align*}
    R_T^{(2)} = \mcal{O}\left(d\cdot \sqrt{T}\log(T) + \kappa \cdot d^2 \cdot \log(T)^2\right)
\end{align*}
\end{restatable}

The formal proof is deferred to Section~\ref{sec:proof_regret_our_bandit} in the supplementary materials. It mostly relies on the following Lemma, which ensures that the first term of $\epsilon_{t,2}(x,\theta^{(2)}_t)$ cumulates sub-linearly and independently of $\kappa$ (up to a second order term that grows only as $\log(T)$).

\begin{restatable}[]{lemma}{bonuscumulatethetabar}
\label{lemma:bonuscumulatethetabar}
        Let $T\geq 1$. Under the event $E_\delta$:
    \begin{align*}
        \sum_{t=1}^T \dot\mu(x_t^\transp\theta^{(2)}_t)\left\lVert x_t\right\rVert_{\mbold{H}_t^{-1}(\theta^{(2)}_t)} \leq &C_4\sqrt{T}+C_5M\kappa \gamma_{T}(\delta)
    \end{align*}
    where $C_4$ and $C_5$ are independent of $\kappa$.
\end{restatable}

\begin{algorithm}[tb]
   \caption{\ourbandit}
   \label{algo:ourbandit}
\begin{algorithmic}
   \STATE {\bfseries Input:} regularization parameter $\lambda$
   \STATE Initialize the set of admissible log-odds $\mcal{W}_0=\Theta$
   \FOR{$t\geq 1$}
   \STATE Compute $\theta^{(2)}_{t}$ (Equation~\eqref{eq:definionthetabar})
   \STATE Observe the contexts-action feature set $\mcal{X}_t$.
    \STATE Play $x_t =  \argmax_{x\in\mcal{X}_t} \mu(x^\transp\theta^{(2)}_t) + \epsilon_{t,2}(x,\theta^{(2)}_t)$.
    \STATE Observe rewards $r_{t+1}$.\\
    \STATE Compute the log-odds $\ell_t = \sup_{\theta'\in\mcal{C}_t(\delta)\cap\Theta} x_t^\transp\theta'$.
    \STATE Add the new constraint to the feasible set: $$\mcal{W}_{t+1}=\mcal{W}_t\cap\{\theta:\,-\ell_t \leq \theta^\transp x_t\leq \ell_t\}.$$
   \ENDFOR
\end{algorithmic}
\end{algorithm}

\begin{sketchofproof}
The proof relies on the fact that $\theta^{(2)}_t\in\mcal W_t$. Intuitively, this allows us to lower-bound $\mbold{H}_t(\theta_t^{(2)})$ by the matrix $\sum_{s=1}^{t-1}\min_{\theta\in\mcal{C}_s(\delta)\cap\Theta}\dot{\mu}(\theta^\transp x_s)x_sx_s^\transp+\lambda\mbold{I}_d$. Note that in this case, $\min_{\theta\in\mcal{C}_s(\delta)\cap\Theta}\dot{\mu}(\theta^\transp x_s)$ is no longer a function of $t$. This, coupled with a one-step Taylor expansion of $\dot\mu$ allows us to use the Elliptical Lemma on a well chosen quantity and obtain the announced rates. The formal proof is deferred to Section~\ref{sec:bonuscumulatethetabar} in the supplementary materials.
\end{sketchofproof}

\section{Discussion}
\label{sec:discussion}
In this work, we studied the scaling of optimistic logistic bandit algorithms for a particular GLM: the logistic model. We explicitly showed that previous algorithms suffered from prohibitive scaling introduced by the quantity $\kappa$, because of their sub-optimal treatment of the non-linearities of the reward signal. Thanks to a novel non-linear approach, we proved that they can be improved by deriving tighter prediction bounds. By doing so, we gave a rigorous justification for an algorithm that resembles the heuristic algorithm empirically evaluated in \cite{filippi2010parametric}. This algorithm exhibits a regret bound that only suffers from a $\sqrt{\kappa}$ dependency, compared to $\kappa$ for previous guarantees. Further, we showed that a more careful algorithmic design leads to yet better guarantees, where the leading term of the regret is independent of $\kappa$. This result bridges the gap between logistic bandits and linear bandits, up to polynomial terms in constants of interest (e.g $S$).

The theoretical value of the regret upper-bound of \ourbandit{} can be highlighted by comparing it to the Bayesian regret lower bound provided by \citet{dongon2019}. Namely, they show that for any logistic bandit algorithm, and for any polynomial form $p$ and $\epsilon>0$, there exist a problem instance such that the regret is at least $\Omega(p(d)T^{1-\epsilon})$. This does not contradict our bound, as for hard problem instance one can have $\kappa = T$ in which case the second term of \ourbandit{} will scale as $d^2T$. 
Note that other corner-cases instances further highlight the theoretical value of our regret bounds. Namely, note that $\kappa=\sqrt{T}$ turns GLM-UCB's regret guarantee vacuous as it will scale linearly with $T$. On the other hand for this case the regret of \glmimproved{} scales as $T^{3/4}$, and the regret of \ourbandit{} continues to scale as $\sqrt{T}$.
    
\paragraph{Extension to other GLMs.} An important avenue for future work consists in extending our results to other generalized linear models. This can be done naturally by extending our work. Indeed, the properties of the sigmoid that we leverage are rather weak, and might easily transfer to other inverse link functions. We first used the fact that $\dot{\mu}$ represents the variance of the reward in order to use Theorem~\ref{thm:bernsteinselfnormalized}. This is not a specificity of the logistic model, but is a common relationship observed for all exponential models and their related mean function \citep[\S 2]{filippi2010parametric}. We also used the generalized self-concordance property of the logistic loss, which is a consequence of the fact that $\vert \ddot{\mu}\vert\leq \dot{\mu}$. This control is quite mild, and other mean functions might display similar properties (with other constants). This is namely the case of another generalized linear model: the Poisson regression.
    
\paragraph{Randomized algorithms.} The lessons learned here for optimistic algorithms might be transferred to randomized algorithms (such as Thompson Sampling) that are often preferred in practical applications thanks to their superior empirical performances. Extending our approach to such algorithms would therefore be of significant practical importance.

\newpage

\bibliography{bibliography.bib}
\bibliographystyle{apalike}

\newpage 

\onecolumn
\appendix

\section*{Organization of the appendix}

This appendix is organized as follows: 
\begin{itemize}
    \item Appendix~\ref{sec:proof_bernsteinselfnormalized} gives the formal proof of our new tail-inequality for self-normalized martingales.
    \item Appendix~\ref{sec:proof_prediction} contains the proof of the concentration and prediction results that are claimed in Section~\ref{sec:prediction}.
    \item Appendix~\ref{sec:proof_regret} provides the formal proof of the regret upper-bounds for \glmimproved and \ourbandit. 
    \item Appendix~\ref{sec:proof_useful_lemmas} contains some useful Lemmas.

\end{itemize}

\section{Proof of Theorem~\ref{thm:bernsteinselfnormalized}}
\label{sec:proof_bernsteinselfnormalized}
\bernsteinselfnormalized*

For readability concerns, we define $\beta=\sqrt{2\lambda}$ and rewrite:
\begin{align*}
   \mbold{H}_t &= \sum_{s=1}^{t-1}\sigma_s^2x_sx_s^T+\frac{\beta^2}{2}\mbold{I}_d\\
   &= \bar{\mbold{H}}_t+ \frac{\beta^2}{2}\mbold{I}_d.
\end{align*}
where $\bar{\mbold{H}}_t \defeq \sum_{s=1}^{t-1}\sigma_s^2x_sx_s^T$.
For all $\xi\in\mbb{R}^d$ let $M_0(\xi)=1$ and for $t\geq 1$ define:
$$
    M_t(\xi) \defeq \exp\left(\xi^TS_t -\left\lVert \xi\right\rVert_{\mbold{\bar{H}}_t}^2\right)
$$
We now claim Lemma~\ref{lemma:martingale} which will be proven later (Section~\ref{sec:prooflemmamartingale}).
\begin{restatable}[]{lemma}{martingale}
For all $\xi\in\mcal{B}_2(d)$, $\{M_t(\xi)\}_{t=1}^\infty$ is a non-negative super-martingale.
\label{lemma:martingale}
\end{restatable}

Our analysis follows the steps of the pseudo-maximization principle introduced in \cite{de2004self}, used by \cite{abbasi2011improved} for the linear bandit and thoroughly detailed in Chapter 20 of \cite{lattimore2018bandit}. The main difference in our analysis come from the restriction $\xi\in\mcal{B}_2(d)$ (instead of $\xi\in\mbb{R}^d$) which calls for some refinements when using the Laplace trick to provide a high-probability bound on the maximum of $\log M_t(\xi)$.

Let $h(\xi)$ be a probability density function with support on $\mcal{B}_2(d)$ (to be defined later). For $t\geq 0$ let: 
$$
    \bar{M}_t \defeq \int_{\xi} M_t(\xi)dh(\xi)
$$
By Lemma 20.3 of \cite{lattimore2018bandit} $\bar{M}_t$ is also a non-negative super-martingale, and $\mbb{E}\left[\bar{M}_0\right]=1$. 
Let $\tau$ be a stopping time with respect to the filtration $\left\{F_t\right\}_{t=0}^\infty$. We can follow the proof of Lemma 8 in \cite{abbasi2011improved} to justify that $\bar{M}_\tau$ is well-defined (independently of whether $\tau<\infty$ holds or not) and that $\mbb{E}\left[\bar{M}_\tau\right]\leq 1$.
 Therefore, with $\delta\in(0,1)$ and thanks to the maximal inequality:
\begin{align}
    \mbb{P}\left(\log(\bar{M}_\tau)\geq \log(\frac{1}{\delta})\right) = \mbb{P}\left(\bar{M}_\tau\geq \frac{1}{\delta}\right)\leq \delta
    \label{eq:stoppingtimeprobacontrol}
\end{align}

We now proceed to compute $\bar{M}_t$ (more precisely a lower bound on $\bar{M}_t$). Let $\beta$ be a strictly positive scalar, and set $h$ to be the density of an isotropic normal distribution of precision $\beta^2$ truncated on $\mcal{B}_2(d)$. We will denote $N(h)$ its normalization constant. Simple computations show that:
$$
    \bar{M}_t = \frac{1}{N(h)}\int_{\mcal{B}_2(d)} \exp\left(\xi^TS_t-\left\lVert \xi\right\rVert_{\mbold{H}_t}^2\right)d\xi
$$
To ease notations, let $f(\xi)\defeq \xi^TS_t-\left\lVert\xi\right\rVert_{\mbold{H}_t}^2$ and $\xi_*=\argmax_{\ltwo{\xi}\leq 1/2} f(\xi)$. Because:
$$
    f(\xi) = f(\xi_*)+(\xi-\xi_*)^T\nabla f(\xi_*)-(\xi-\xi_*)^T\mbold{H}_t(\xi-\xi_*)
$$
we obtain that:
\begin{align*}
    \bar{M}_t 
    &= \frac{e^{f(\xi_*)}}{N(h)}\int_{\mbb{R}^d}\indicator{\ltwo{\xi}\leq 1}\exp\left((\xi-\xi_*)^T\nabla f(\xi_*)-(\xi-\xi_*)^T\mbold{H}_t(\xi-\xi_*)\right)d\xi\\
    &= \frac{e^{f(\xi_*)}}{N(h)}\int_{\mbb{R}^d}\indicator{\ltwo{\xi+\xi_*}\leq 1}\exp\left(\xi^T\nabla f(\xi_*)-\xi^T\mbold{H}_t\xi\right)d\xi & \text{(change of variable $\xi+\xi_*$)}\\
    &\geq \frac{e^{f(\xi_*)}}{N(h)}\int_{\mbb{R}^d}\indicator{\ltwo{\xi}\leq 1/2}\exp\left(\xi^T\nabla f(\xi_*)-\xi^T\mbold{H}_t\xi\right)d\xi & \text{(as $\ltwo{\xi_*} \leq 1/2$)}\\
    &= \frac{e^{f(\xi_*)}}{N(h)}\int_{\mbb{R}^d}\indicator{\ltwo{\xi}\leq 1/2}\exp\left(\xi^T\nabla f(\xi_*)\right)\exp\left(-\frac{1}{2}\xi^T(2\mbold{H}_t)\xi\right)d\xi\\
\end{align*}
By defining $g(\xi)$ the density of the normal distribution of precision $2\mbold{H}_t$ truncated on the ball $\left\{\xi\in\mbb{R}^d, \ltwo{\xi}\leq 1/2\right\}$ and noting $N(g)$ its normalizing constant, we can rewrite:
\begin{align}
    \bar{M}_t 
    &\geq \exp\left(f(\xi_*)\right)\frac{N(g)}{N(h)}\mbb{E}_{g}\left[\exp\left(\xi^T\nabla f(\xi_*)\right)\right]\nonumber\\
    &\geq \exp\left(f(\xi_*)\right)\frac{N(g)}{N(h)}\exp\left(\mbb{E}_g\left[\xi^T\nabla f(\xi_*)\right]\right) & \text{(Jensen's inequality)}\nonumber\\
    &\geq \exp\left(f(\xi_*)\right)\frac{N(g)}{N(h)} & \text{(as $\mbb{E}_g\left[\xi\right]=0$)}\label{eq:proof_thm_1_final_m_t}
\end{align}
Unpacking this results and assembling \eqref{eq:stoppingtimeprobacontrol} and \eqref{eq:proof_thm_1_final_m_t}, we obtain that for any $\xi_0$ such that $\ltwo{\xi_0}\leq 1/2$:
\begin{align}
        \mbb{P}\left(\bar{M}_t\geq \frac{1}{\delta}\right)&\geq \mbb{P}\left(\exp\left(f(\xi_*)\right)\frac{N(g)}{N(h)}\geq 1/\delta\right)\nonumber \\
        &=\mbb{P}\left( \log\left(\exp\left(f(\xi_*)\right)\frac{N(g)}{N(h)}\right)\geq \log(1/\delta)\right)\nonumber\\
         &=\mbb{P}\left(f(\xi_*)\geq \log(1/\delta)+\log\left(\frac{N(h)}{N(g)}\right)\right)\nonumber\\
         &= \mbb{P}\left(\max_{\ltwo{\xi}\leq 1/2}\xi^TS_t-\left\lVert\xi\right\rVert_{\mbold{H}_t}^2\geq \log(1/\delta)+\log\left(\frac{N(h)}{N(g)}\right)\right)\nonumber\\
         &\geq \mathbb{P}\left( \xi_0^TS_t - \left\lVert\xi_0\right\rVert_{\mbold{H}_t}^2\geq \log(1/\delta)+\log\left(\frac{N(h)}{N(g)}\right)\right)\label{eq:probmaxunderlambda}
\end{align}

In particular, we can use:
\begin{align*}
    \xi_0 \defeq \frac{\mbold{H}_t^{-1}S_t}{\left\lVert S_t\right\rVert_{\mbold{H}_t^{-1}}}\frac{\beta}{2\sqrt{2}}
\end{align*}
since
\begin{align*}
    \ltwo{\xi_0} \leq \frac{\beta}{2\sqrt{2}} \left(\xi_{\text{min}}(\bar{\mbold{H}}_t)+\frac{\beta^2}{2}\right)^{-1/2} \leq 1/2
\end{align*}
Using this value of $\xi_0$ in Equation~\eqref{eq:probmaxunderlambda} yields:
$$
    \mbb{P}\left(\left\lVert S_t\right\rVert_{\mbold{H}_t^{-1}}\geq \frac{\beta}{2\sqrt{2}} +\frac{2\sqrt{2}}{\beta}\log\left(\frac{N(h)}{\delta N(g)}\right)\right)\leq \mbb{P}\left(\bar{M}_t\geq \frac{1}{\delta}\right)
$$
To finish the proof we have left to explicit the quantities $N(h)$ and $N(g)$. Lemma~\ref{lemma:NhoverNg} provides an upper-bound for the log of their ratio. Its proof is given in Section~\ref{sec:prooflemmaratio}.
\begin{restatable}[]{lemma}{NhoverNg}
    The following inequality holds:
    $$
        \log\left(\frac{N(h)}{N(g)}\right) \leq \log\left(\frac{2^{d/2}\det\left(\mbold{H}_t\right)^{1/2}}{\beta^{d}}\right)+d\log(2)
    $$
\label{lemma:NhoverNg}
\end{restatable}

Therefore with probability at least $1-\delta$ and by using Equation~\eqref{eq:stoppingtimeprobacontrol}:

\begin{align}
    \left\lVert S_\tau\right\rVert_{\mbold{H}_\tau^{-1}} 
    &
        \leq \frac{\beta}{2\sqrt{2}}
        +
        \frac{2\sqrt{2}}{\beta}\log(1/\delta)
        + 
         \frac{2\sqrt{2}}{\beta}\log\left(\frac{2^{d/2}\det\left(\mbold{H}_\tau\right)^{1/2}}{\beta^{d}}\right)
        +
        \frac{2\sqrt{2}}{\beta}d\log(2)
\end{align}
Directly following the stopping time construction argument in the proof of Theorem 1 of \cite{abbasi2011improved} we obtain that  with probability at least $1-\delta$, for all $t\in\mbb{N}$:
\begin{align}
    \left\lVert S_t\right\rVert_{\mbold{H}_t^{-1}} \leq \frac{\beta}{2\sqrt{2}}+\frac{2\sqrt{2}}{\beta}\log\left(\frac{2^{d/2}\det\left(\mbold{H}_t\right)^{1/2}}{\beta^{d}\delta}\right)+\frac{2\sqrt{2}}{\beta}d\log(2)
\end{align}
Finally, recalling that $\beta=\sqrt{2\lambda}$ provides the announced result.

\subsection{Proof of Lemma~\ref{lemma:martingale}}
\label{sec:prooflemmamartingale}

\martingale*
\begin{proof}
For all $t\geq 1$ we have that:
\begin{align*}
    \mbb{E}\left[\exp(\xi^T S_t)\vert \mcal{F}_{t-1}\right]= \exp(\xi^T S_{t-1})\mbb{E}\left[\exp(\xi^Tx_{t-1}\varepsilon_{t})\vert \mcal{F}_{t-1}\right].
\end{align*}
Since $\vert \xi^T x_{t-1}\vert\leq 1$ all conditions of Lemma~\ref{lemma:momentgeneratingcontrol} (stated and proven below) are checked and:
\begin{align*}
    \mbb{E}\left[\exp(\xi^T S_t)\vert \mcal{F}_{t-1}\right] &\leq  \exp(\xi^T S_{t-1})(1+\sigma_{t-1}^2 (x_{t-1}^T\xi)^2) &\\
    &\leq \exp(\xi^T S_{t-1}+\sigma_{t-1}^2 (x_{t-1}^T\xi)^2) &(1+x\leq e^x)
\end{align*}
Therefore:
\begin{align*}
    \mbb{E}\left[M_t(\xi)\vert \mcal{F}_{t-1}\right]
    & = \mbb{E}\left[\exp\left(\xi^T S_t - \mnorm{\xi}{H_t}^2\right)\middle\vert \mcal{F}_{t-1}\right]\\
    & = \mbb{E}\left[\exp\left(\xi^T S_t\right)\middle\vert \mcal{F}_{t-1}\right]\exp\left(-\sum_{s=1}^{t-1}\sigma_s^2(x_s^T\xi)^2\right)\\
    &\leq \exp\left(\xi^TS_{t-1} +\sigma_{t-1}^2 (x_{t-1}^T\xi)^2 -\sum_{s=1}^{t-1}\sigma_s^2(x_s^T\xi)^2\right)\\
    & \leq M_{t-1}(\xi)
\end{align*}
yielding the announced result.
\end{proof}
To prove Lemma~\ref{lemma:martingale} we needed the following result.
\begin{lemma}
Let $\varepsilon$ be a centered random variable of variance $\sigma^2$ and such that $\vert \varepsilon\vert \leq 1$ almost surely. Then for all $\lambda\in[-1,1]$:
$$
    \mbb{E}\left[\exp(\lambda \epsilon)\right] \leq 1 + \lambda^2\sigma^2.
$$
\label{lemma:momentgeneratingcontrol}
\end{lemma}
\begin{proof}
A decomposition of the exponential gives:
$$
\begin{aligned}
     \mbb{E}\left[\exp(\lambda \epsilon)\right] &= 1 + \sum_{k=2}^{\infty} \frac{\lambda^k}{k!}\mbb{E}\left[\epsilon^{k-2}\epsilon^2\right] &\\
     &\leq 1+ \lambda^2\sigma^2\sum_{k=2}^{\infty} \frac{\vert\lambda\vert^{k-2}}{k!} &(\vert \varepsilon\vert \leq 1)\\
     & \leq 1+\lambda^2\sigma^2(e-2)\leq 1+\lambda^2\sigma^2 &(\vert \lambda\vert\leq 1)
\end{aligned}
$$
\end{proof}

\subsection{Proof of Lemma~\ref{lemma:NhoverNg}}
\label{sec:prooflemmaratio}
\NhoverNg*
By definition of $N(h)$ and thanks to a change of variable: 
\begin{align*}
    N(h) &= \int_{\mbb{R}^d} \mathds{1}\left[\lVert \xi\rVert_2 \leq 1\right]\exp\left(-\frac{1}{2}\beta^2\lVert \xi\rVert_2^2\right)d\xi\\
    &=  \beta^{-d}\int_{\mbb{R}^d} \mathds{1}\left[\lVert \xi\rVert_2 \leq \beta\right]\exp\left(-\frac{1}{2}\lVert \xi\rVert_2^2\right)d\xi
\end{align*}
Also by a change of variable:
\begin{align*}
    N(g) &=  \int_{\mbb{R}^d} \mathds{1}\left[\lVert \xi\rVert_2 \leq 1/2\right]\exp\left(-\frac{1}{2}\xi^T(2\mbold{H}_t)\xi\right)d\xi \\
     &= \det(\mbold{H}_t)^{-1/2}2^{-d/2} \int_{\mbb{R}^d} \mathds{1}\left[\lVert 2^{-1/2}\mbold{H}_t^{-1/2}\xi\rVert_2 \leq 1/2\right]\exp\left(-\frac{1}{2}\lVert \xi\rVert_2^2\right)\\
     &= \det(\mbold{H}_t)^{-1/2}2^{-d/2} \int_{\mbb{R}^d} \mathds{1}\left[\left\lVert \left(\bar{\mbold{H}}_t+\frac{\beta^2}{2}\mbold{I}_d\right)^{-1/2}\xi\right\rVert_2 \leq 1/\sqrt{2}\right]\exp\left(-\frac{1}{2}\lVert \xi\rVert_2^2\right)\\
     &\geq \det(\mbold{H}_t)^{-1/2}2^{-d/2} \int_{\mbb{R}^d} \mathds{1}\left[\lVert\xi\rVert_2 \leq \beta/2\right]\exp\left(-\frac{1}{2}\lVert \xi\rVert_2^2\right)\\    
\end{align*}
We obtain the following upper-bound on the ratio $N(h)/N(g)$:
\begin{align}
    \frac{N(h)}{N(g)} \leq \beta^{-d}\det(\mbold{H}_t)^{1/2}\,2^{d/2}\,\frac{\int_{\mbb{R}^d} \mathds{1}\left[\lVert \xi\rVert_2 \leq \beta\right]\exp\left(-\frac{1}{2}\lVert \xi\rVert_2^2\right)d\xi}{\int_{\mbb{R}^d} \mathds{1}\left[\lVert\xi\rVert_2 \leq \beta/2\right]\exp\left(-\frac{1}{2}\lVert \xi\rVert_2^2\right)}
    \label{eq:NhNginterm}
\end{align}
Note that:
\begin{align*}
    \frac{\int_{\mbb{R}^d} \mathds{1}\left[\lVert \xi\rVert_2 \leq \beta\right]\exp\left(-\frac{1}{2}\lVert \xi\rVert_2^2\right)d\xi}{\int_{\mbb{R}^d} \mathds{1}\left[\lVert\xi\rVert_2 \leq \beta/2\right]\exp\left(-\frac{1}{2}\lVert \xi\rVert_2^2\right)} &= 1 + \frac{\int_{\mbb{R}^d} \mathds{1}\left[\beta/2\leq \lVert \xi\rVert_2 \leq \beta\right]\exp\left(-\frac{1}{2}\lVert \xi\rVert_2^2\right)d\xi}{\int_{\mbb{R}^d} \mathds{1}\left[\lVert\xi\rVert_2 \leq \beta/2\right]\exp\left(-\frac{1}{2}\lVert \xi\rVert_2^2\right)}\\
    &\leq 1 + \frac{\exp\left(-\frac{1}{8}\beta^2\right)}{\exp\left(-\frac{1}{8}\beta^2\right)} \cdot \frac{\int_{\mbb{R}^d} \mathds{1}\left[\beta/2\leq \lVert \xi\rVert_2 \leq \beta\right]d\xi}{\int_{\mbb{R}^d} \mathds{1}\left[\lVert\xi\rVert_2 \leq \beta/2\right]}\\
    &= 1+ \frac{\mcal{V}_d(\beta)-\mcal{V}_d(\beta/2)}{\mcal{V}_d(\beta/2)}\\
    &= 2^d
\end{align*}
where $\mcal{V}_d(\beta)\propto \beta^d$ denotes the volume of the $d$-dimensional ball of radius $\beta$. Plugging this result in Equation~\eqref{eq:NhNginterm} and taking the logarithm yields the announced result:
\begin{align*}
    \log\left(\frac{N(h)}{N(g)}\right) \leq \log\left(\frac{2^{d/2}\det(\mbold{H}_t)^{1/2}}{\beta^d}\right) + d\log(2)
\end{align*}

\section{Proofs of prediction and concentration results}
\label{sec:proof_prediction}

For all this section, we will use the following notations:
\begin{align*}
    \alpha(x,\theta_1,\theta_2) &\defeq \int_{v=0}^1 \dot{\mu}(vx^\transp \theta_2+(1-v)x^\transp\theta_1)dv >0 \\
    \mbold{G}_t(\theta_1,\theta_2) &\defeq \sum_{s=1}^{t-1} \alpha(x,\theta_1,\theta_2)x_sx_s^\transp + \lambda\mbold{I}_d \\
    \mbold{H}_t(\theta_1) &\defeq \sum_{s=1}^{t-1}\dot{\mu}(x^\transp \theta_1)x_sx_s^\transp + \lambda \mbold{I}_d\\
    \mbold{V}_t &\defeq \sum_{s=1}^{t-1}x_sx_s^\transp + \kappa\lambda \mbold{I}_d
\end{align*}
where $\theta_1,\theta_2$ and $x$ are vectors in $\mbb{R}^d$ and $\lambda$ is a strictly positive scalar. We will extensively use that fact that $\forall \theta_1,\theta_2\in\Theta$ we have $\mbold{G}_t(\theta_1,\theta_2)\geq \kappa^{-1}\mbold{V}_t$ and $\mbold{H}_t(\theta_1)\geq \kappa^{-1}\mbold{V}_t$.

The quantities $\alpha(x,\theta_1,\theta_2)$ and $\mbold{G}_t(\theta_1,\theta_2)$ naturally arise when studying GLMs. Indeed, note that for all $x\in\mbb{R}^d$ and $\theta\in\mbb{R}^d$, the following equality holds:
    \begin{align}
        \mu(x^\transp \theta_1)-\mu(x^\transp \theta_2) = \alpha(x,\theta_2,\theta_1)x^\transp (\theta_1-\theta_2)
        \label{eq:mvt}
    \end{align}
This result is classical (see \cite{filippi2010parametric}) and can be obtained by a straight-forward application of the mean-value theorem. It notably allows us to link $\theta_1-\theta_2$ with $g_t(\theta_1)-g_t(\theta_2)$. Namely, it is straightforward that:
\begin{align*}
    g_t(\theta_1)-g_t(\theta_2) &= \sum_{s=1}^{t-1}\alpha(x_s,\theta_2,\theta_1) x_sx_s^\transp(\theta_1-\theta_2) +\lambda\theta_1 - \lambda \theta_2\\
    &= \mbold{G}_t(\theta_2,\theta_1)(\theta_1-\theta_2)
\end{align*}
Because $\mbold{G}_t(\theta_1,\theta_2)\succ \mbold{0}_{d\times d}$ this yields:
\begin{align}
    \left \lVert\theta_1-\theta_2\right\rVert_{\mbold{G}_t(\theta_2,\theta_1)} = \left \lVert g_t(\theta_1)-g_t(\theta_2)\right\rVert_{\mbold{G}_t^{-1}(\theta_2,\theta_1)}
    \label{eq:dThetatodgt}
\end{align}

\subsection{Proof of Lemma~\ref{lemma:confidenceset}}
\label{sec:proofconcentration}

\lemmaconfidenceset*

Note that, for any $t\geq 1$:
\begin{align}
    \theta_*\in\mcal{C}_t \Longleftrightarrow \left\lVert g_t(\hat{\theta}_t) - g_t(\theta_*)\right\rVert_{\mbold{H}_t^{-1}(\theta_*)} \leq \gamma_t(\delta)
    \label{eq:equivalenceconfset}
\end{align}
To prove Lemma~\ref{lemma:confidenceset} we therefore need to ensure that the r.h.s happens  for all $t\geq 1$ with probability at least $1-\delta$. This is the object of the following Lemma, of which Lemma~\ref{lemma:confidenceset} is a direct corollary.

\begin{restatable}[]{lemma}{StHtbounded}
Let $\delta\in(0,1]$. With probability at least $1-\delta$:
\begin{align*}
    \forall t\geq 1, \quad \left\lVert g_t(\hat{\theta}_t) - g_t(\theta_*)\right\rVert_{\mbold{H}_t^{-1}(\theta_*)} \leq \gamma_t(\delta)
\end{align*}
\label{lemma:StHtbounded}
\end{restatable}
\begin{proof}
Recall that $\hat{\theta}_t$ is the unique maximizer of the log-likelihood:
\begin{align*}
    \mcal{L}_t^\lambda(\theta) &\defeq \sum_{s=1}^{t-1} \Big[r_{s+1}\log\mu(x_s^\transp \theta)+(1-r_{s+1})\log(1-\mu(x_s^\transp \theta))\Big] - \frac{\lambda}{2}\ltwo{\theta}^2
\end{align*}
and therefore $\hat{\theta}_t$ is a critical point of $\mcal{L}_t^\lambda(\theta)$. Solving for $\nabla_\theta \mcal{L}_t^\lambda = 0$ and using the fact that $\dot{\mu}=\mu(1-\mu)$ we obtain:
\begin{align*}
    \sum_{s=1}^{t-1}\mu(\hat{\theta}_t^\transp x_s)x_s +\lambda\hat{\theta}_t= \sum_{s=1}^{t-1}r_{s+1}x_s
\end{align*}
This result, combined with the definition of $g_t(\theta_*) = \sum_{s=1}^{t-1}\mu(x_s^\transp \theta_*)x_s+\lambda\theta_*$ yields:
$$
\begin{aligned}
    g_t(\hat{\theta}_t) - g_t(\theta_*) &= \sum_{s=1}^{t-1} \varepsilon_{s+1}x_s - \lambda \theta_*\\
    &= S_t - \lambda\theta_*
\end{aligned}
$$
where we denoted $\varepsilon_{s+1}\defeq r_{s+1}-\mu(x_s^\transp \theta_*)$ for all $s\geq 1$ and $S_t \defeq \sum_{s=1}^{t-1}\varepsilon_{s+1}x_s$ for all $t\geq 1$. Simple linear algebra implies that:
\begin{align}
    \left\lVert g_t(\hat{\theta}_t) - g_t(\theta_*)\right\rVert_{\mbold{H}_t^{-1}(\theta_*)} \leq   \left\lVert S_t\right\rVert_{\mbold{H}_t^{-1}(\theta_*)} + \sqrt{\lambda} S
    \label{eq:firstboundgt}
\end{align}
Note that by Equation~\eqref{eq:reward_model}, $\{\varepsilon_t\}_{t=1}^\infty$ is a martingale difference sequence adapted to $\mcal{F}$ and almost surely bounded by 1. Also, note that for all $s\geq 1$:
$$
    \dot{\mu}(x_s^\transp \theta_*)=\mu(x_s^\transp \theta_*)(1-\mu(x_s^\transp \theta_*)) = \mbb{E}\left[\varepsilon_{s+1}^2\vert \mcal{F}_t\right]\rdefeq \sigma_s^2 
$$
and thus $\mbold{H}_t(\theta_*)=\sum_{s=1}^{t-1}\sigma_s^2x_sx_s^\transp  + \lambda\mbold{I}_d$. All the conditions of Theorem~\ref{thm:bernsteinselfnormalized} are checked and therefore:
\begin{align}
  1-\delta &\leq \mbb{P}\left(\forall t\geq 1, \left\lVert S_t \right\rVert_{\mbold{H}_t^{-1}(\theta_*)}\leq \frac{\sqrt{\lambda}}{2} +  \frac{2}{\sqrt{\lambda}}\log\left(\frac{\det(\mbold{H}_t(\theta_*))^{1/2}\lambda^{-d/2}}{\delta}\right)+ \frac{2d}{\sqrt{\lambda}}\log(2)\right)\nonumber\\
  &\leq \mbb{P}\left(\forall t\geq 1, \left\lVert S_t \right\rVert_{\mbold{H}_t^{-1}(\theta_*)}\leq \frac{\sqrt{\lambda}}{2} +  \frac{2}{\sqrt{\lambda}}\log\left(\frac{\left(\lambda +L t/d\right)^{d/2}\lambda^{-d/2}}{\delta}\right)+ \frac{2d}{\sqrt{\lambda}}\log(2)\right)\nonumber\\
  &\leq \mbb{P}\left(\forall t\geq 1, \left\lVert S_t \right\rVert_{\mbold{H}_t^{-1}(\theta_*)}\leq \frac{\sqrt{\lambda}}{2} +  \frac{2}{\sqrt{\lambda}}\log\left(\frac{\left(1 +\frac{ L t}{\lambda d}\right)^{d/2}}{\delta}\right)+ \frac{2d}{\sqrt{\lambda}}\log(2)\right)\nonumber
  \\
  &= \mbb{P}\left(\forall t\geq 1, \left\lVert S_t \right\rVert_{\mbold{H}_t^{-1}(\theta_*)}\leq \gamma_t(\delta)-\sqrt{\lambda}S \right)\label{eq:defEdelta}
\end{align}
where we used that:
\begin{align*}
    \det\left(\mbold{H}_t(\theta_*)\right) \leq  L^{d}\det\left(\sum_{s=1}^{t-1}x_sx_s^\transp  + \frac{\lambda}{ L}\mbold{I}_d\right)\leq  L^d \left(\frac{\lambda}{ L} + \frac{T}{d}\right)^d\leq \left(\lambda + \frac{ Lt}{d}\right)^d
\end{align*}
thanks to Lemma~\ref{lemma:determinant_trace_inequality}.
Assembling Equation~\eqref{eq:firstboundgt} with Equation~\eqref{eq:defEdelta} yields:
\begin{align*}
    \mbb{P}\left(\forall t\geq 1, ~\left\lVert g_t(\hat{\theta}_t) - g_t(\theta_*)\right\rVert_{\mbold{H}_t^{-1}(\theta_*)} \leq \gamma_t(\delta)\right)
    &\geq \mbb{P}\left( \forall t\geq 1, ~\left\lVert S_t\right\rVert_{\mbold{H}_t^{-1}(\theta_*)} + \sqrt{\lambda} S\leq \gamma_t(\delta)\right)\\
    &\geq 1-\delta
\end{align*}
hence the announced result.
\end{proof}
\begin{rem*}
In the following sections we will often use the rewriting of $E_\delta$ inherited from Equation~\eqref{eq:equivalenceconfset}:
\begin{align*}
    E_\delta = \left\{ \forall t\geq 1, ~~ \left\lVert g_t(\hat{\theta}_t) - g_t(\theta_*)\right\rVert_{\mbold{H}_t^{-1}(\theta_*)} \leq \gamma_t(\delta) \right\}
\end{align*}
\end{rem*}
\subsection{Key self-concordant results}
We start this section by claiming and proving Lemma~\ref{lemma:self_concordance}, which uses the generalized self-concordance property of the log-loss and will be later used to derive useful lower-bounds on the function $\alpha(\cdot)$.
\begin{lemma}[Self-concordance control]
\label{lemma:self_concordance}
For any $z_1,z_2\in\mathbb{R}$, we have the following inequality:
$$
    \dot{\mu}(z_1)\frac{1-\exp(-\vert z_1-z_2\vert)}{\vert z_1-z_2\vert} \leq \int_{0}^1 \dot{\mu}(z_1+v(z_2-z_1))dv \leq  \dot{\mu}(z_1)\frac{\exp(\vert z_1-z_2\vert)-1}{\vert z_1-z_2\vert}
$$
Furthermore:
$$
    \int_{0}^1 \dot{\mu}(z_1+v(z_2-z_1))dv \geq \dot{\mu}(z_1)(1+\vert z_1-z_2\vert)^{-1}
$$
\end{lemma}
\begin{proof}
    The proof is based on the generalized self-concordance property of the logistic loss, which is detailed and exploited in other works on the logistic regression \citep{bach2010self}. This part of the analysis relies on similar properties of the logistic function. Indeed, a short computation shows that for any $z\in\mbb{R}^d$, one has $\vert \ddot{\mu}(z)\vert \leq \dot{\mu}(z)$. Therefore, using the fact that $\dot{\mu}(z)>0$ for all $z$ in any compact of $\mbb{R}$, one has that for all $z\geq z_1$:
    $$
        -(z-z_1) \leq \int_{z_1}^{z}\frac{d}{dv}\log\dot{\mu}(v)dv  \leq z-z_1
    $$
which in turns gives us that:
$$
    \dot{\mu}(z_1)\exp(-(z-z_1)) \leq \dot{\mu}(z) \leq \dot{\mu}(z_1)\exp(z-z_1)
$$
Assuming that $z_2\geq z_1$, setting $z=z_1+v(z_2-z_1)>z_1$ for $v\in[0,1]$ and integrating over $v$ gives:
$$
    \dot{\mu}(z_1) \frac{1-\exp(-(z_2-z_1))}{z_2-z_1} \leq \int_{0}^1 \dot{\mu}(z_1+v(z_2-z_1))dv \leq \dot{\mu}(z_1)\frac{\exp(z_2-z_1)-1}{z_2-z_1}
$$
Repeating this operation for $z\leq z_1$ and $z_2\leq z_1$ leads to:
$$
    \dot{\mu}(z_1) \frac{\exp(z_2-z_1)-1}{z_2-z_1} \leq \int_{0}^1 \dot{\mu}(z_1+v(z_2-z_1))dv \leq \dot{\mu}(z_1)\frac{1-\exp(z_1-z_2)}{z_1-z_2}
$$
Combining the last two equations gives the first result. Note that if $x\geq 0$ we have $e^{-x}\leq (1+x)^{-1}$, and therefore $(1-e^{-x})/x \geq (1+x)^{-1}$. Applying this inequality to the l.h.s of the first result provides the second statement of the lemma.
\end{proof}

We now state Lemma~\ref{lemma:boundGtbyHt} which allows to provide a control of $\mbold{G}_t(\theta_1,\theta_2)$ by $\mbold{H}_t(\theta_1)$ and $\mbold{H}_t(\theta_2)$. 

\begin{restatable}[]{lemma}{boundGtbyHt}
For all $\theta_1,\theta_2\in\Theta$ the following inequalities hold:
\begin{align*}
    \mbold{G}_t(\theta_1,\theta_2) &\geq (1+2S)^{-1}\mbold{H}_t(\theta_1)\\
    \mbold{G}_t(\theta_1,\theta_2) &\geq (1+2S)^{-1}\mbold{H}_t(\theta_2)
\end{align*}
\label{lemma:boundGtbyHt}
\end{restatable}

\begin{proof}
The proof relies on the self-concordance property of the log-loss, which comes with the fact $\vert \ddot{\mu}\vert \leq \dot{\mu}$. As detailed in Lemma~\ref{lemma:self_concordance} this allows us to derive an exponential-control lower bound on $\dot{\mu}$. Let $x\in\mcal{B}_2(d)$. By applying Lemma~\ref{lemma:self_concordance} with $z_1=x^\transp \theta_1$ and $z_2=x^\transp (\theta_2-\theta_1)$ we obtain that:
\begin{align*}
        \alpha(x,\theta_1,\theta_2) &\geq \left(1+\left\vert x^\transp (\theta_1-\theta_2)\right\vert\right)^{-1} \dot{\mu}(x^\transp \theta_1) & \\
        &\geq \left(1+\ltwo{x}\ltwo{\theta_1-\theta_2}\right)^{-1}\dot{\mu}(x^\transp \theta_1) &\text{(Cauchy-Schwartz)}\\
        &\geq \left(1+2S\right)^{-1}\dot{\mu}(x^\transp \theta_1) & \text{($\theta_1,\theta_2\in\Theta, x\in\mcal{B}_2(d)$)}
\end{align*}
Using this lower bound:
\begin{align*}
    \mbold{G}_t(\theta_1,\theta_2) &= \sum_{s=1}^{t-1} \alpha(x_s,\theta_1,\theta_2)x_sx_s^\transp  + \lambda\mbold{I}_d
    \\ &\succeq (1+2S)^{-1}\sum_{s=1}^{t-1}\dot{\mu}(x_s^\transp \theta_1)x_sx_s^\transp  + \lambda\mbold{I}_d &(x_s\in\mcal{X}\subseteq\mcal{B}_2(d), \theta_1,\theta_2\in\Theta)\\
    &= (1+2S)^{-1}\left(\sum_{s=1}^{t-1}\dot{\mu}(x_s^\transp \theta_1)x_sx_s^\transp  + (1+2S)\lambda \mbold{I}_d\right)\\
    &\succeq (1+2S)^{-1}\left(\sum_{s=1}^{t-1}\dot{\mu}(x_s^\transp \theta_1)x_sx_s^\transp  + \lambda \mbold{I}_d\right) & (S>0)\\
    &=(1+2S)^{-1}\mbold{H}_t(\theta_1)
\end{align*}
which yields the first inequality. Realizing (through a change of variable for instance) that $\theta_1$ and $\theta_2$ have symmetric roles in the definition of $\alpha(x,\theta_1,\theta_2)$ directly yields the second inequality.
\end{proof}

\subsection{Proof of claims in Section~\ref{sec:newconfidenceset}}
\label{sec:proofconfidenceset}

This section focuses on giving a rigorous proof for the different "degraded" confidence sets that we introduce for visualization purposes. For a rigorous proof, we need to discard the assumption that $\hat\theta_t\in\Theta$. To this end we introduce the "projections":
\begin{align*}
    \theta_t^{\rm L} &\defeq \argmin_{\theta\in\Theta} \left\lVert g_t(\theta)-g_t(\hat\theta_t)\right\rVert_{\mbold{V}_t^{-1}},\\
    \theta_t^{\rm NL} &\defeq \argmin_{\theta\in\Theta} \left\lVert g_t(\theta)-g_t(\hat\theta_t)\right\rVert_{\mbold{H}_t^{-1}(\theta)}.
\end{align*}
Note that when $\hat\theta_t\in\Theta$, both $\theta_t^{\rm L}$ and $\theta_t^{\rm NL}$ are equal to $\hat\theta_t$.
We properly define $\mcal{E}_t^{\rm NL}(\delta)$ using the estimator $\theta_t^{\rm NL}$:
\begin{align*}
    \mcal{E}_t^{\rm NL}(\delta) &\defeq \left\{\theta\in\Theta, \left\lVert \theta - \theta_t^{\rm NL}\right\rVert_{\mbold{H}_t(\theta)} \leq (2+4S)\gamma_t(\delta)\right\}.
\end{align*}
When $\hat\theta_t\in\Theta$ we can save a factor 2 in the width of the set, hence formally matching the definition we gave in the main text. 

\begin{restatable}[]{lemma}{elltnl}
With probability at least $1-\delta$:
\begin{align*}
    \forall t\geq1,~~ \theta_*\in \mcal{E}_t^{\rm NL}(\delta)
\end{align*}
\label{lemma:elltnl}
\end{restatable}

\begin{proof}

\begin{align}
        \left \lVert \theta_* - \theta_t^{\rm NL}\right\rVert_{\mbold{H}_t(\theta_*)} &\leq \sqrt{1+2S}\left \lVert \theta_* - \theta_t^{\rm NL}\right\rVert_{\mbold{G}_t(\theta_*,\theta_t^{\rm NL})} & (\text{Lemma~\ref{lemma:boundGtbyHt}},\, \theta_*,\theta_t^{\rm NL}\in\Theta) \nonumber\\
        &=\sqrt{1+2S}\left \lVert g_t(\theta_*) - g_t(\theta_t^{\rm NL})\right\rVert_{\mbold{G}_t^{-1}(\theta_*,\theta_t^{\rm NL})} &\text{(Equation~\eqref{eq:dThetatodgt})}\nonumber\\
        &= \sqrt{1+2S}\left \lVert g_t(\theta_*) + g_t(\hat{\theta}_t)-g_t(\hat{\theta}_t)- g_t(\theta_t^{\rm NL})\right\rVert_{\mbold{G}_t^{-1}(\theta_*,\theta_t^{\rm NL})}\nonumber\\
        &\leq \!\sqrt{1\!+\!2S}\!\left(\!\left \lVert g_t(\hat{\theta}_t) \!-\! g_t(\theta_t^{\rm NL})\right\rVert_{\mbold{G}_t^{-1}(\theta_*,\theta_t^{\rm NL})}\!+\!\left \lVert g_t(\hat{\theta}_t) \!-\! g_t(\theta_*)\right\rVert_{\mbold{G}_t^{-1}(\theta_*,\theta_t^{\rm NL})}\right)\nonumber\\
        &\leq (1+2S)\left(\left \lVert g_t(\hat{\theta}_t) - g_t(\theta_t^{\rm NL})\right\rVert_{\mbold{H}_t^{-1}(\theta_t^{\rm NL})}+\left \lVert g_t(\hat{\theta}_t) - g_t(\theta_*)\right\rVert_{\mbold{H}_t^{-1}(\theta_*)}\right) &(\text{Lemma~\ref{lemma:boundGtbyHt}},\, \theta_*,\theta_t^{\rm NL}\in\Theta)\nonumber\\
        &\leq 2(1+2S) \left \lVert g_t(\hat{\theta}_t) - g_t(\theta_*)\right\rVert_{\mbold{H}_t^{-1}(\theta_*)}&\text{(definition of }\theta_t^{\rm NL}) &\nonumber\\
        &\leq (2+4S)\gamma_t(\delta) &\text{(Lemma~\ref{lemma:StHtbounded})}\nonumber
\end{align}
where the last line holds with probability at least $1-\delta$. Therefore, with probability at least $1-\delta$ we have shown that $\theta_*\in\mcal{E}_t^{\rm NL}(\delta)$ for all $t\geq 1$ which concludes the proof.
\end{proof}

Note that we used the fact that:
\begin{align*}
    \mbold{G}_t(\theta_*,\theta_t^{\rm NL})&\geq (1+2S)^{-1}\mbold{H}_t(\theta_*)\\
    \mbold{G}_t(\theta_*,\theta_t^{\rm NL})&\geq (1+2S)^{-1}\mbold{H}_t(\theta_t^{\rm NL})
\end{align*}
which is inherited from Lemma~\ref{lemma:boundGtbyHt}, itself a consequence of the self-concordance property of the log-loss. This allows to replace the matrix $\mbold{G}_t(\theta_*,\theta_t^{\rm NL})$ by $\mbold{H}_t(\theta_*)$ which still conveys \emph{local} information and allows us to use Theorem~\ref{thm:bernsteinselfnormalized} through Lemma~\ref{lemma:StHtbounded}. As we shall see next, another candidate to replace $\mbold{G}_t(\theta_*,\theta_t^{\rm NL})$ is $\mbold{V}_t$, however at the loss of local information for global information, which consequently adding a dependency in $\kappa$ instead of $S$. 

We now properly define $\mcal{E}_t^{\rm L}(\delta)$ using the estimator $\theta_t^{\rm L}$:
\begin{align*}
    \mcal{E}_t^{\rm L}(\delta) &\defeq \left\{\theta\in\Theta, \left\lVert \theta - \theta_t^{\rm L}\right\rVert_{\mbold{V}_t} \leq 2\kappa\beta_t(\delta)\right\}
\end{align*}
where $\beta_t\defeq \sqrt{\lambda}S + \sqrt{\log(1/\delta)+2d\log\left(1+\frac{ t}{\kappa \lambda d}\right)}$.
Again, when $\hat\theta_t\in\Theta$ this formally matches the definitions we gave in the main text (up to a factor 2, which can be eliminated when $\hat{\theta}_t\in\Theta$).
\begin{lemma}
    With probability at least $1-\delta$:
    \begin{align*}
        \forall t\geq 1, ~~ \theta_*\in\mcal{E}_t^{\rm L}(\delta).
    \end{align*}
\end{lemma}

\begin{proof}
    Since:
        \begin{align*}
             \left \lVert \theta_* - \theta_t^{\rm L}\right\rVert_{\mbold{V}_t} &\leq  \sqrt{\kappa}\left \lVert \theta_* - \theta_t^{\rm L}\right\rVert_{\mbold{G}_t(\theta_*,\theta_t^{\rm L})}  &\left(\mbold{V}_t \leq \kappa\mbold{G}_t(\theta_*,\theta_t^{\rm L})\right)\\
             &\leq \sqrt{\kappa}\left\lVert g_t(\theta_t^{\rm L})-g_t(\theta_*)\right\rVert_{\mbold{G}_t^{-1}(\theta_*,\theta_t^{\rm L})} &\text{(mean value theorem)}\\
             &\leq \kappa \left\lVert g_t(\theta_t^{\rm L})-g_t(\theta_*)\right\rVert_{\mbold{V}_t^{-1}} & \left(\mbold{G}_t^{-1}(\theta_*,\theta_t^{\rm L}) \leq \kappa\mbold{V}_t^{-1}\right)\\
             &\leq \kappa \left(\left\lVert g_t(\hat\theta_t)-g_t(\theta_*)\right\rVert_{\mbold{V}_t^{-1}}+\left\lVert g_t(\hat\theta_t)-g_t(\theta_t^{\rm L})\right\rVert_{\mbold{V}_t^{-1}}\right) &\\
             &\leq 2\kappa \left\lVert g_t(\hat\theta_t)-g_t(\theta_*)\right\rVert_{\mbold{V}_t^{-1}} &(\text{definition of $\theta_t^{\rm L}$ }, \theta_*\in\Theta)\\
             &\leq 2\kappa\left(\sqrt{\lambda}S + \left\lVert S_t \right\rVert_{\mbold{V}_t^{-1}}\right) &\\
             &\leq 2\kappa\left(\sqrt{\lambda}S + \sqrt{\log(1/\delta)+2d\log(1+\frac{t}{\kappa\lambda d}}\right) &\text{(Theorem 1 of \cite{abbasi2011improved})}
        \end{align*}
    where the last line holds for all $t\geq 1$ on an event of probability at least $1-\delta$. This means that with probability at least $1-\delta$, $\theta_*\in\mcal{E}_{t}(\delta)$ for all $t\geq 1$ which finishes the proof.
\end{proof}
Based on this confidence sets, one can derive results on the prediction error similar to those announced in \citep[Appendix A.2, Proposition 1]{filippi2010parametric}. Indeed, for all $x\in\mcal{X}, \theta\in\mcal{E}_t^{\rm L}(\delta)$, and under the event  $\left\{\theta_*\in\mcal{E}_t^{\rm NL}(\delta), ~~ \forall t\geq 1\right\}$, which holds with probability at least $1-\delta$:
\begin{align*}
    \Delta^{\rm pred}(x,\theta)  &= \alpha(x,\theta_*,\theta)\vert x^T(\theta-\theta_*)\vert &\text{(mean-value theorem)}\\
    &\leq L\vert x^T(\theta-\theta_*)\vert&(\dot\mu\leq L)\\
    & = L\vert x^T\mbold{V}_t^{-1/2}\mbold{V}_t^{1/2}(\theta-\theta_*)\vert \\
    &\leq L\left\lVert x\right\rVert_{\mbold{V}_t^{-1}}\left\lVert \theta-\theta_*\right\rVert_{\mbold{V}_t} &\text{(Cauchy-Schwartz)}\\
    &\leq L\left\lVert x\right\rVert_{\mbold{V}_t^{-1}}\left(\left\lVert \theta-\theta_t^{\rm L}\right\rVert_{\mbold{V}_t}+\left\lVert \theta_t^{\rm L}-\theta_*\right\rVert_{\mbold{V}_t}\right) &\\
    &\leq 4L\kappa\left\lVert x\right\rVert_{\mbold{V}_t^{-1}}\beta_t(\delta) &\theta,\theta_*\in\mcal{E}_t^{\rm NL}(\delta)
\end{align*}

\subsection{Proof of Lemma~\ref{lemma:lemmapredictionbound1}}
\label{sec:prooflemmapredictionbound1}
\lemmapredictionboundone*

\begin{proof}
During this proof we work under the good event $E_\delta$, which holds with probability at least $1-\delta$.

For all $t\geq 1$, for all $x\in\mcal{X}$, for all $\theta\in\mcal{C}_t(\delta)$:
\begin{align*}
    \Delta^{\text{pred}}(x,\theta) &= \left\vert \mu(x^\transp \theta_*)-\mu(x^\transp \theta)\right\vert &\\ &\leq \alpha(x,\theta_*,\theta)\left\vert x^\transp (\theta_*-\theta)\right\vert &\text{(Equation~\eqref{eq:mvt})}\\ 
    &\leq L\left\vert x^\transp (\theta_*-\theta)\right\vert & (\dot\mu\leq L)\\
    &\leq L\left\vert x^\transp \mbold{G}_t^{-1/2}(\theta,\theta_*)\mbold{G}_t^{1/2}(\theta,\theta_*)(\theta_*-\theta)\right\vert& \\
    &\leq L\left\lVert x\right\rVert_{\mbold{G}_t^{-1}(\theta,\theta_*)}\left\lVert \theta_*-\theta\right\rVert_{\mbold{G}_t(\theta,\theta_*)} &\text{(Cauchy-Schwartz)}\\
    &\leq L\sqrt{\kappa}\left\lVert x\right\rVert_{\mbold{V}_t^{-1}}\left\lVert \theta_*-\theta\right\rVert_{\mbold{G}_t(\theta,\theta_*)} &(\mbold{G}_t(\theta, \theta_*)\geq \kappa^{-1}\mbold{V}_t) \\
    &= L\sqrt{\kappa}\left\lVert x\right\rVert_{\mbold{V}_t^{-1}}\left\lVert g_t(\theta_*)-g_t(\theta)\right\rVert_{\mbold{G}_t^{-1}(\theta,\theta_*)} &\text{(Equation~\eqref{eq:dThetatodgt})}\\
     &=L\sqrt{\kappa}\left\lVert x\right\rVert_{\mbold{V}_t^{-1}}\left(\left\lVert g_t(\theta_*)-g_t(\hat\theta_t)\right\rVert_{\mbold{G}_t^{-1}(\theta,\theta_*)}+\left\lVert g_t(\hat\theta_t)-g_t(\theta)\right\rVert_{\mbold{G}_t^{-1}(\theta,\theta_*)}\right) &\\
      &\leq L\sqrt{1+2S}\sqrt{\kappa}\left\lVert x\right\rVert_{\mbold{V}_t^{-1}}\left(\left\lVert g_t(\theta_*)-g_t(\hat\theta_t)\right\rVert_{\mbold{H}_t^{-1}(\theta_*)}+\left\lVert g_t(\hat\theta_t)-g_t(\theta)\right\rVert_{\mbold{H}_t^{-1}(\theta)}\right) &\text{(Lemma~\ref{lemma:boundGtbyHt}, $\theta,\theta_*\in\Theta$)}\\
      &\leq L\sqrt{1+2S}\sqrt{\kappa}\left\lVert x\right\rVert_{\mbold{V}_t^{-1}}(\gamma_t(\delta)+\gamma_t(\delta)) &(\theta,\theta_*\in\mcal{C}_t(\delta))\\
      &\leq 2L\sqrt{1+2S}\sqrt{\kappa}\left\lVert x\right\rVert_{\mbold{V}_t^{-1}}\gamma_t(\delta)
\end{align*}
which proves the desired result.
\end{proof}

The main device of this proof is the application of Lemma~\ref{lemma:boundGtbyHt}, itself inherited from the self-concordance of the log-loss. This allows to replace the matrix $\mbold{G}_t(\theta_*,\theta)$ by $\mbold{H}_t(\theta_*)$ and $\mbold{H}_t(\theta)$, at the price of a $\sqrt{1+2S}$ multiplicative factor (instead of $\sqrt{\kappa}$ when we lower-bound $\mbold{G}_t(\theta_*,\theta)$ with $\kappa\mbold{V}_t$). However, following this line of proof we loose two times the local information carried by $\theta$; the first time when using the fact that $\alpha(x,\theta,\theta_*)\leq L$, the second time when upper-bounding $\left\lVert x\right\rVert_{\mbold{G}_t^{-1}(\theta_*,\theta)}$ by $\sqrt{\kappa}\left\lVert x\right\rVert_{\mbold{V}_t^{-1}}$. This flaws are corrected in the more careful analysis leading to Lemma~\ref{lemma:predictionboundtwo}.

\subsection{Proof of Lemma~\ref{lemma:predictionboundtwo}}
\label{sec:prooflemmapredictionbound2}

\lemmapredictiontwo*

\begin{proof}
During this proof we work under the good event $E_\delta$, which holds with probability at least $1-\delta$.

By a first-order Taylor expansion one  has :
\begin{align*}
    \alpha(x,\theta_*,\theta) &\leq \dot{\mu}(x^\transp \theta) + M \left\vert x^\transp (\theta-\theta_*)\right\vert & (\vert \ddot{\mu}\vert \leq M)\\
    & \leq \dot{\mu}(x^\transp \theta) +  M \left\lVert x\right\rVert_{\mbold{G}_t^{-1}(\theta,\theta_*)}\left\lVert \theta-\theta_* \right\rVert_{\mbold{G}_t(\theta,\theta_*)} &\text{(Cauchy-Schwartz)}\\
    &\leq  \dot{\mu}(x^\transp \theta) +  M \left\lVert x\right\rVert_{\mbold{G}_t^{-1}(\theta,\theta_*)}\left\lVert g_t(\theta)-g_t(\theta_*) \right\rVert_{\mbold{G}_t^{-1}(\theta,\theta_*)} &\text{(Equation~\eqref{eq:dThetatodgt})}
\end{align*}

It can be extracted from the proof of Lemma~\ref{lemma:lemmapredictionbound1} in Section~\ref{sec:prooflemmapredictionbound1} that for all $t\geq 1$, $\theta\in\Theta$, $x\in\mcal{X}$:
\begin{align*}
    \Delta^{\rm{pred}}(x,\theta) \leq \alpha(x,\theta_*,\theta)\left\lVert x\right\rVert_{\mbold{G}_t^{-1}(\theta,\theta_*)}\left\lVert g_t(\theta)-g_t(\theta_*) \right\rVert_{\mbold{G}_t^{-1}(\theta,\theta_*)}
\end{align*}
and also that:
\begin{align*}
    \left\lVert g_t(\theta)-g_t(\theta_*) \right\rVert_{\mbold{G}_t^{-1}(\theta,\theta_*)} \leq 2\sqrt{1+2S}\gamma_t(\delta)
\end{align*}
Unpacking these results, we get that:
\begin{align*}
    \Delta^{\text{pred}}(x,\theta) &\leq \dot{\mu}(x^T\theta)\sqrt{4+8S}\left\lVert x\right\rVert_{\mbold{G}_t^{-1}(\theta,\theta_*)}\gamma_t(\delta) + M(4+8S)\left\lVert x\right\rVert_{\mbold{G}_t^{-1}(\theta,\theta_*)}^2\gamma_t(\delta)^2\\
    &\leq \dot{\mu}(x^T\theta)\sqrt{4+8S}\left\lVert x\right\rVert_{\mbold{G}_t^{-1}(\theta,\theta_*)}\gamma_t(\delta) + M(4+8S)\kappa\left\lVert x\right\rVert_{\mbold{V}_t^{-1}}^2\gamma_t(\delta)^2 & (\mbold{G}_t(\theta_*,\theta)\geq \kappa^{-1}\mbold{V}_t)\\
    &\leq \dot{\mu}(x^T\theta)(2+4S)\left\lVert x\right\rVert_{\mbold{H}_t^{-1}(\theta)}\gamma_t(\delta) + M(4+8S)\kappa\left\lVert x\right\rVert_{\mbold{V}_t^{-1}}^2\gamma_t(\delta)^2 & (\mbold{G}_t(\theta_*,\theta)\geq (1+2S)^{-1}\mbold{H}_t(\theta))\\
\end{align*}
which proves the desired result.
\end{proof}

\subsection{Proof of Lemma~\ref{lemma:bonuscumulatethetabar}}
\label{sec:bonuscumulatethetabar}

We here claim a result more general than Lemma~\ref{lemma:bonuscumulatethetabar}. The latter is actually a direct corollary of the following Lemma, once one has checked that $\theta_t^{(2)}\in\mcal{C}_t(\delta)$ for all $t\geq 1$ (which will be formally proven in the following Section).

\begin{lemma}
Let $T\geq 1$. Under the event $E_\delta$, for all sequence $\{\theta_t\}_{t=1}^T$ such that $\theta_t\in\mcal{C}_t(\delta)\cap \mcal{W}_t$:
\begin{align*}
    \sum_{t=1}^T \dot\mu(x_t^\transp \theta_t)\left\lVert x_t\right\rVert_{\mbold{H}_t^{-1}(\theta)}\leq C_4\sqrt{T} + C_5M\kappa\gamma_T(\delta),
\end{align*}
where the constants
\begin{align*}
    C_4 &= \sqrt{2L\max(1,L/\lambda)}\sqrt{d\log\left(1+\frac{LT}{d\lambda}\right)}~~~,\\
    C_5 &= 4d\sqrt{1+2S}\max(1,1/(\kappa\lambda))\log\left(1+\frac{T}{\kappa d\lambda}\right)
\end{align*}
show no dependencies in $\kappa$ (except in logarithmic terms).
\label{lemma:bonuscumulates}
\end{lemma}
\begin{proof}
During this proof we work under the good event $E_\delta$, which holds with probability at least $1-\delta$.

We start the proof by making the following remark. Note that $\mcal{W}_t$ can be rewritten as:
\begin{align}
    \mcal{W}_t = \left\{\theta\in\Theta \text{ s.t } \dot{\mu}(x_s^\transp \theta)\geq \inf_{\theta'\in\mcal{C}_s(\delta)\cap\Theta} \dot{\mu}(x_s^\transp \theta'), \, \forall s \leq t-1\right\}.
    \label{eq:newdefinitionWt}
\end{align}
Indeed, using the monotonicity of $\dot{\mu}$ on $\mbb{R}_d^+$ and $\mbb{R}_d^-$, one can show that this re-writting is equivalent with the one provided in the main text:
\begin{align*}
    \mcal{W}_t   = \left\{\theta\in\Theta \text{ s.t } \vert x_s^\transp \theta \vert \leq \sup_{\theta'\in\mcal{C}_s(\delta)\cap\Theta} \vert x_s^\transp \theta'\vert, \, \forall s \leq t-1\right\}.
\end{align*}
Further, note on the high-probability event $E_\delta$ we have $\theta_*\in\mcal{C}_t(\delta)\cap \Theta$ for all $t\geq 1$. This implies that under $E_\delta$ we have $\theta_*\in\mcal{W}_t$ for all $t\geq 1$ (as a result, the set $\{\mcal{W}_t\}_t$ are therefore not empty). 

In the following, we will use the following notation:
\begin{align*}
    \theta'_t &\defeq \argmin_{\theta\in \mcal{C}_t(\delta)\cap\Theta} \dot\mu(x_t^\transp \theta).
\end{align*}
First, for all $\theta_t\in\mcal{W}_t$:
$$
\begin{aligned}
    \mbold{H}_t(\theta_t) &= \sum_{s=1}^{t-1} \dot{\mu}(x_s^\transp \theta_t)x_sx_s^\transp  + \lambda\mbold{I}_d \\
    &\geq \sum_{s=1}^{t-1} \inf_{\theta\in\mcal{C}_s(\delta)\cap\Theta}\dot{\mu}(x_s^\transp \theta)x_sx_s^\transp  + \lambda\mbold{I}_d &(\theta_t\in\mcal{W}_t \text{ and Equation~\eqref{eq:newdefinitionWt}})\\
    &= \sum_{s=1}^{t-1} \dot{\mu}(x_s^\transp \theta'_s)x_sx_s^\transp  + \lambda\mbold{I}_d \defeq \mbold{L}_t&
\end{aligned}
$$
Therefore for all $x\in\mcal{X}$:
\begin{align}
    \left\lVert x\right\rVert_{\mbold{H}_t^{-1}(\theta_t)} \leq  \left\lVert x \right\rVert_{\mbold{L}_t^{-1}}.
\label{eq:boundxnormhtbar}
\end{align}
Also, a first-order Taylor expansion gives that for all $x\in\mcal{X}$ and $\theta_t\in\mcal{C}_t(\delta)\cap\mcal{W}_t$: 
\begin{align*}
    \dot{\mu}(x^\transp \theta_t) &\leq \dot{\mu}(x^\transp \theta'_t) +  M \left\vert  x^\transp (\theta_t-\theta'_t)\right\vert &\\
    &\leq  \dot{\mu}(x^\transp \theta'_t) +  M \left\lVert x\right\rVert_{\mbold{G}_t^{-1}(\theta_t,\theta'_t)}\left\lVert \theta_t-\theta'_t \right\rVert_{\mbold{G}_t(\theta_t,\theta'_t)} & \text{(Cauchy-Schwartz)}\\
     &\leq  \dot{\mu}(x^\transp \theta'_t) +  M \left\lVert x\right\rVert_{\mbold{G}_t^{-1}(\theta_t,\theta'_t)}\left\lVert g_t(\theta_t)-g_t(\theta'_t) \right\rVert_{\mbold{G}_t^{-1}(\theta_t,\theta'_t)} & \text{(Equation~\eqref{eq:dThetatodgt})} \\
     &\leq \dot{\mu}(x^\transp \theta'_t) +  M\sqrt{\kappa} \left\lVert x\right\rVert_{\mbold{V}_t^{-1}}\left\lVert g_t(\theta_t)-g_t(\theta'_t) \right\rVert_{\mbold{G}_t^{-1}(\theta_t,\theta'_t)} & (\mbold{G}_t(\theta_t, \theta'_t)\geq \kappa^{-1}\mbold{V}_t) \\ 
     &\leq \dot{\mu}(x^\transp \theta'_t) +  2M\sqrt{\kappa}\sqrt{1+2S} \left\lVert x\right\rVert_{\mbold{V}_t^{-1}}\gamma_t(\delta) &(\theta'_t,\theta_t\in\mcal{C}_t(\delta))
\end{align*}
where we used that:
\begin{align*}
    \left\lVert g_t(\theta_t)-g_t(\theta'_t) \right\rVert_{\mbold{G}_t^{-1}(\theta_t,\theta'_t)} &\leq  \left\lVert g_t(\theta_t)-g_t(\hat{\theta}_t) \right\rVert_{\mbold{G}_t^{-1}(\theta_t,\theta'_t)} +  \left\lVert g_t(\hat{\theta}_t)-g_t(\theta'_t) \right\rVert_{\mbold{G}_t^{-1}(\theta_t,\theta'_t)}\\
     &\leq \sqrt{1+2S}\left(\left\lVert g_t(\theta_t)-g_t(\hat{\theta}_t) \right\rVert_{\mbold{H}_t^{-1}(\theta_t)} +  \left\lVert g_t(\hat{\theta}_t)-g_t(\theta'_t) \right\rVert_{\mbold{H}_t^{-1}(\theta'_t)}\right) &( \text{Lemma~\ref{lemma:boundGtbyHt}})\\
     &\leq 2\sqrt{1+2S}\gamma_t(\delta) &(\theta_t,\theta_t'\in\mcal{C}_t(\delta))
\end{align*}
Unpacking, we obtain that for all $x\in\mcal{X}$, $\theta_t\in\mcal{C}_t(\delta)\cap\mcal{W}_t$:
\begin{align}
     \dot{\mu}(x^\transp \theta_t)\left\lVert x\right\rVert_{\mbold{H}_t^{-1}(\theta_t)} &\leq \dot{\mu}(x^\transp \theta'_t)\left\lVert x\right\rVert_{\mbold{H}_t^{-1}(\theta_t)} +  2M\sqrt{\kappa}\sqrt{1+2S}\gamma_{t}(\delta)\left\lVert x\right\rVert_{\mbold{H}_t^{-1}(\theta_t)}\left\lVert x\right\rVert_{\mbold{V}_t^{-1}}\nonumber\\
     &\leq \dot{\mu}(x^\transp \theta'_t)\left\lVert x\right\rVert_{\mbold{H}_t^{-1}(\theta_t)} +  2M\kappa\sqrt{1+2S}\gamma_{t}(\delta)\left\lVert x\right\rVert_{\mbold{V}_t^{-1}}^2 &(\mbold{H}_t(\theta_t)\geq \kappa^{-1}\mbold{V}_t)\nonumber\\
     &\leq \dot{\mu}(x^\transp \theta'_t)\left\lVert x\right\rVert_{\mbold{L}_t^{-1}} +  2M\kappa\sqrt{1+2S}\gamma_{t}(\delta)\left\lVert x\right\rVert_{\mbold{V}_t^{-1}}^2 &\text{(Equation~\eqref{eq:boundxnormhtbar})}\label{eq:cumulbonusfirtandsecondterm}
\end{align}

We first study how the first term of the r.h.s cumulates. For the trajectory $\left\{x_t\right\}_{t=1}^T $ let us denote $\tilde{x}_t\defeq \sqrt{\dot{\mu}(x_t^\transp \theta'_t)} x_t$ for every $1\leq t \leq T$. Note that:
\begin{align*}
    \mbold{L}_t = \sum_{s=1}^{t-1} \dot{\mu}(x_s^\transp \theta'_s)x_sx_s^\transp  + \lambda\mbold{I}_d = \sum_{s=1}^{t-1} \tilde{x}_s\tilde{x}_s^\transp  + \lambda\mbold{I}_d
\end{align*}
which means that we can apply the Elliptical Lemma to $\sum_{t=1}^{T}\left\lVert \tilde{x_t}\right\rVert_{\mbold{L}_t^{-1}}^2$. Hence:
\begin{align*}
    \sum_{t=1}^\transp \dot{\mu}(x_t^\transp \theta'_t)\left\lVert x_t\right\rVert_{\mbold{L}_t^{-1}} &\leq \sqrt{ L} \sum_{t=1}^\transp  \left\lVert \tilde{x}_t\right\rVert_{\mbold{L}_t^{-1}}&\\
    &\leq \sqrt{ L}\sqrt{T}\sqrt{\sum_{t=1}^\transp  \left\lVert \tilde{x}_t\right\rVert_{\mbold{L}_t^{-1}}^2} &\text{(Cauchy-Schwartz)}\\
    &\leq\sqrt{ L T}\sqrt{2\max(1,L/\lambda)}\sqrt{\log\left(\det\left(\mbold{L}_{T+1}\right)\lambda^{-d}\right)} &\text{(Lemma~\ref{lemma:ellipticalpotential})}\\
    &\leq\sqrt{ L T}\sqrt{2\max(1,L/\lambda)}\sqrt{d\log\left(1+\frac{ L T}{d\lambda}\right)} &\text{(Lemma~\ref{lemma:determinant_trace_inequality})}
\end{align*}

For the second term, a second application of Lemma~\ref{lemma:ellipticalpotential} and Lemma~\ref{lemma:determinant_trace_inequality} gives:
\begin{align*}
    \sum_{t=1}^T \gamma_t(\delta)\left\lVert x_t\right\rVert_{\mbold{V}_t^{-1}}^2 &\leq \gamma_T(\delta)\sum_{t=1}^\transp \left\lVert x_t\right\rVert_{\mbold{V}_t^{-1}}^2 &(t:\to \gamma_t(\delta) \text{ increasing})\\
    &\leq 2d\gamma_T(\delta)\max(1,1/(\kappa\lambda))\log\left(1+\frac{T }{\kappa d\lambda}\right)
\end{align*}
Assembling these last two inequalities and Equation~\eqref{eq:cumulbonusfirtandsecondterm} yields the announced result.
\end{proof}

\section{Regret proof}
\label{sec:proof_regret}
\subsection{Regret decomposition}
\label{subsec:regret_decomposition}
The pseudo-regret at round $T$ is:
$$
    R_T = \sum_{s=1}^T \mu(\theta_*^Tx_*^t) - \mu(\theta_*^Tx_t)
$$
We will consider \emph{optimistic} algorithms, that is algorithms that at round $t$, for a given estimator $\theta_t$ of $\theta_*$ and a given exploration bonus $\epsilon_t(x)$ plays the action
$$
    x_t = \argmax_{x\in\mcal{X}_t} \mu(\theta_t^Tx) + \epsilon_t(x)
$$
The following Lemma characterizes the regret of such an algorithm. 
\begin{restatable}[]{lemma}{regretdecomposition}
For any $T\geq 1$:
$$
    R_T\leq \sum_{t=1}^T \Delta^{\rm pred}(x_t,\theta_t) +  \sum_{t=1}^{T} \epsilon_t(x_t)
    + \sum_{t=1}^T \Delta^{\rm pred}(x_*^t, \theta_t)-\sum_{t=1}^T \epsilon_t(x_*^t)
$$
\label{lemma:regretdecomposition}
\end{restatable}
\begin{proof}
By removing and adding $\sum_{t=1}^T \mu(\theta_t^Tx_*^t)$ and $\sum_{t=1}^T \mu(\theta_t^Tx_t)$ one has that:
$$
\begin{aligned}
    R_T &= \sum_{t=1}^T \mu(\theta_*^Tx_*^t) - \mu(\theta_*^Tx_t) = \left[\sum_{t=1}^T \mu(\theta_*^Tx_*^t)-\mu(\theta_t^Tx_*^t)\right] + \left[\sum_{t=1}^T \mu(\theta_t^Tx_*^t)-\mu(\theta_t^Tx_t)\right]+ \left[\sum_{t=1}^T \mu(\theta_t^Tx_t)-\mu(\theta_*^Tx_t)\right]\\
    &\leq \sum_{t=1}^T \Delta^{\text{pred}}(x_t,\theta_t) + \sum_{t=1}^T \Delta ^{\text{pred}}(x_*^t,\theta_t) + \left[\sum_{t=1}^T \mu(\theta_t^Tx_*^t)-\mu(\theta_t^Tx_t)\right]
\end{aligned}
$$
Note that by definition of $x_t$:
$$
    \mu(x_t^T\theta_t) + \epsilon_t(x_t)\geq \mu(\theta_t^T x_*^t) + \epsilon_t(x_*^t)
$$

which yields the announced results. 

\end{proof}

Note that under the assumption that for all $t\geq 1$ and all $x\in\mcal{X}$, \ul{if we have} $\epsilon_t(x) \geq \Delta^{\text{pred}}(x,\theta_t)$ Lemma~\ref{lemma:regretdecomposition} directly yields that:
\begin{align*}
    R_T \leq 2\sum_{t=1}^T \epsilon_t(x_t)
\end{align*}

\subsection{Proof of Theorem~\ref{thm:regretofglmimproved}}
\label{sec:proofregretofglmimproved}
\regretofglmimproved*
\begin{proof}
During this proof we work under the good event $E_\delta$, which holds with probability at least $1-\delta$ (as shown in Lemma~\ref{lemma:confidenceset}). 

We start by showing that $\theta_t^{(1)}\in\mcal{C}_t(\delta)$ for all $t\geq 1$. Indeed: 
\begin{align*}
    \left\lVert g_t(\theta_t^{(1)})-g_t(\hat\theta_t) \right\rVert_{\mbold{H}_t^{-1}(\theta_t^{(1)})} &= \min_{\theta\in\Theta}  \left\lVert g_t(\theta)-g_t(\hat\theta_t) \right\rVert_{\mbold{H}_t^{-1}(\theta)} &\text{(definition of $\theta_t^{(1)}$)} \\
    &\leq  \left\lVert g_t(\theta_*)-g_t(\hat\theta_t) \right\rVert_{\mbold{H}_t^{-1}(\theta_*)} &(\theta_*\in\Theta)\\
    &\leq \gamma_t(\delta) &\text{(Lemma~\ref{lemma:StHtbounded}, $E_\delta$ holds.)}
\end{align*}
which proves the desired result.

As $\theta_t^{(1)}\in\mcal{C}_t$ we have from Lemma~\ref{lemma:lemmapredictionbound1} that for all $t\geq 1$ and $x\in\mcal{X}$:
\begin{align*}
    \Delta^{\text{pred}}(x,\theta_t^{(1)}) &\leq L\sqrt{4+8S}\sqrt{\kappa}\gamma_t(\delta)\left\lVert x\right\rVert_{\mbold{V}_t^{-1}}\\
    &= \epsilon_{t,1}(x)
\end{align*}
and therefore by Lemma~\ref{lemma:regretdecomposition} we have that:
\begin{align*}
    R_T &\leq 2\sum_{t=1}^T \epsilon_{t,1}(x_t)&\\
    &= 4L\sqrt{1+2S}\sqrt{\kappa}\sum_{t=1}^T \gamma_{t}(\delta)\left\lVert x_t\right\rVert_{\mbold{V}_t^{-1}} \\
    &\leq 4L\sqrt{1+2S}\sqrt{\kappa}\gamma_{T}(\delta)\sum_{t=1}^T \left\lVert x_t\right\rVert_{\mbold{V}_t^{-1}}&\left(\gamma_{t}(\delta)\leq \gamma_{T}(\delta)\right)\\
    &\leq 4L\sqrt{1+2S}\sqrt{\kappa}\gamma_{T}(\delta)\sqrt{T}\sqrt{\sum_{t=1}^T\left\lVert x_t\right\rVert_{\mbold{V}_t^{-1}}^2} &\text{(Cauchy-Schwartz)} \\
    &\leq 4L\sqrt{1+2S}\sqrt{\kappa}\gamma_{T}(\delta)\sqrt{T}\sqrt{2\max(1,1/(\kappa\lambda))\log\left(\det(\mbold{V}_{T+1})(\kappa\lambda)^{-d}\right)} &\text{(Lemma~\ref{lemma:ellipticalpotential})}\\
    &\leq L\gamma_{T}(\delta)\sqrt{\kappa}\sqrt{T}\sqrt{32d(1+2S)\max(1,(\kappa\lambda))\log\left(1+\frac{T}{\kappa\lambda d}\right)} &\text{(Lemma~\ref{lemma:determinant_trace_inequality})}
\end{align*}
which concludes the proof of the first statement. The second is immediate when realizing that when $\lambda=d\log(T)$ then $\gamma_T(\delta) = \mcal{O}(d^{1/2}\cdot\log(T)^{1/2})$. 
\end{proof}

\subsection{Proof of Theorem~\ref{thm:regretourbandit}}
\regretourbandit*
\label{sec:proof_regret_our_bandit}

\begin{proof}
During this proof we work under the good event $E_\delta$, which holds with probability at least $1-\delta$ (as shown in Lemma~\ref{lemma:confidenceset}). 

We start this proof by showing that $\theta_t^{(2)}\in\mcal{C}_t(\delta)$. Note that under $E_\delta$, we have $\theta_*\in\bigcap_{s=1}^{t-1} \mcal{C}_t(\delta)$ for all $t\geq 1$ and therefore $\theta_*\in\mcal{W}_t$. Further: 
\begin{align*}
    \left\lVert g_t(\theta_t^{(2)})-g_t(\hat\theta_t) \right\rVert_{\mbold{H}_t^{-1}(\theta_t^{(2)})} &= \min_{\theta\in\mcal{W}_t}  \left\lVert g_t(\theta)-g_t(\hat\theta_t) \right\rVert_{\mbold{H}_t^{-1}(\theta)} &\text{(definition of $\theta_t^{(2)}$)} \\
    &\leq  \left\lVert g_t(\theta_*)-g_t(\hat\theta_t) \right\rVert_{\mbold{H}_t^{-1}(\theta_*)} &(\theta_*\in\mcal{W}_t)\\
    &\leq \gamma_t(\delta) &\text{(Lemma~\ref{lemma:StHtbounded}, $E_\delta$ holds.)}
\end{align*}
which proves the desired result.

Therefore $\theta_t^{(2)}\in\mcal{C}_t(\delta)\cap\mcal{W}_t$ and we can apply Lemma~\ref{lemma:predictionboundtwo}. From this we know that:
\begin{align*}
    \Delta^{\text{pred}}(x,\theta_t^{(2)})&\leq (2+4S)\dot{\mu}(x^\transp\theta)\left\lVert x\right\rVert_{\mbold{H}_t^{-1}(\theta)}\gamma_t(\delta) + (4+8S)M\kappa\gamma_t^2(\delta)\left\lVert x\right\rVert_{\mbold{V}_t^{-1}}^2\\
    &= \epsilon_{t,2}(x)
\end{align*}
and therefore thanks to Lemma~\ref{lemma:regretdecomposition} we have:
\begin{align}
    R_T &\leq 2\sum_{t=1}^T \epsilon_{t,2}(x_t)\nonumber\\
    &\leq (4+8S)\sum_{t=1}^T \dot{\mu}(\bar{\theta}_t^Tx_t)\left\lVert x_t\right\rVert_{\mbold{H}_t^{-1}(\bar{\theta}_t)}\gamma_t(\delta) +(8+16S)M\kappa\sum_{t=1}^T\gamma_{t}^2(\delta)\left\lVert x_t\right\rVert_{\mbold{V}_t^{-1}}^2\nonumber\\
    &\leq (4+8S)\gamma_{T}(\delta)\sum_{t=1}^T \dot{\mu}(\bar{\theta}_t^Tx_t)\left\lVert x_t\right\rVert_{\mbold{H}_t^{-1}(\bar{\theta}_t)} + (8+16S)M\kappa\gamma_{T}^2(\delta)\sum_{t=1}^T\left\lVert x_t\right\rVert_{\mbold{V}_t^{-1}}^2\label{eq:regretthetabarfulldecomposition}
\end{align}

Note that:
\begin{align*}
    \sum_{t=1}^T\left\lVert x\right\rVert_{\mbold{V}_{t}^{-1}}^2 &\leq 2\max(1,1/(\kappa\lambda))\log\left(\det(\mbold{V}_{T+1})(\lambda\kappa)^{-d}\right) & \text{(Lemma~\ref{lemma:ellipticalpotential})}\\
    &\leq 2d\max(1,1/(\kappa\lambda))\log\left(1+\frac{T}{\kappa\lambda d}\right) &\text{(Lemma~\ref{lemma:determinant_trace_inequality})}
\end{align*}
and according to Lemma~\ref{lemma:bonuscumulatethetabar} we have:
\begin{align*}
    \sum_{t=1}^T \mu(x_t^T\bar{\theta}_t)\left\lVert x_t\right\rVert_{\mbold{H}_t^{-1}(\bar{\theta}_t)} \leq &\sqrt{ L T}\sqrt{2\max(1,L/\lambda)}\sqrt{d\log\left(1+\frac{ L T}{d\lambda}\right)}\\
        &+4dM\kappa\gamma_{T}(\delta)\sqrt{1+2S}\max(1,1/(\kappa\lambda))\log\left(1+\frac{T }{d\kappa\lambda}\right)
\end{align*}
Assembling these last two inequalities with Equation~\eqref{eq:regretthetabarfulldecomposition} gives:
\begin{align*}
    R_T \leq &(4+8S)\gamma_{T}(\delta)\sqrt{ L T}\sqrt{2d\max(1,L/\lambda)\log\left(1+\frac{ LT}{d\lambda}\right)} \\
    &+ M\kappa d(8+16S)\gamma_{T}^2(\delta)\max(1,1/(\kappa\lambda))\log\left(1+\frac{T }{d\kappa\lambda}\right)(2+2\sqrt{1+2S})
\end{align*}
which concludes the proof of the first statement. The second is immediate when realizing that when $\lambda=d\log(T)$ then $\gamma_T(\delta) = \mcal{O}(d^{1/2}\cdot \log(T)^{1/2})$.
\end{proof}

\section{Useful lemmas}
\label{sec:proof_useful_lemmas}

The following Lemma is a version of the Elliptical Potential Lemma and can be extracted from Lemma 11 in \cite{abbasi2011improved}. We remind its statement and its proof here for the sake of completeness.

\begin{lemma}[Elliptical potential]
    Let $\{x_s\}_{s=1}^\infty$ a sequence in $\mbb{R}^d$ such that $\ltwo{x_s}\leq X$ for all $s\in\mbb{N}$, and  let $\lambda$ be a non-negative scalar. For $t\geq 1$ define $\mbold{V}_t \defeq \sum_{s=1}^{t-1} x_sx_s^T+\lambda\mbold{I}_d$. The following inequality holds:
    $$
        \sum_{t=1}^{T} \left\lVert x_t\right\rVert_{\mbold{V}_t^{-1}}^2 \leq 2\max(1,\frac{X^2}{\lambda})\log\frac{\det(\mbold{V}_{T+1})}{\lambda^d}
   $$
\label{lemma:ellipticalpotential}
\end{lemma}
\begin{proof}
By definition of $\mbold{V}_t$:
\begin{align*}
    \left\vert \mbold{V}_{t+1}\right\vert &= \left\vert \mbold{V}_{t} + x_tx_t^T\right\vert\\
    &= \left\vert \mbold{V}_t\right\vert \left\vert \mbold{I}_d + \mbold{V}_t^{-1/2}x_tx_t^T\mbold{V}_t^{-1/2}\right\vert\\
    &= \left\vert \mbold{V}_t\right\vert\left(1+\left\lVert x_t\right\rVert_{\mbold{V}_t^{-1}}^2\right)\\
\end{align*}
and therefore by taking the log on both side of the equation and summing from $t=1$ to $T$:
\begin{align*}
     \sum_{t=1}^T \log\left(1+\left\lVert x_t\right\rVert_{\mbold{V}_t^{-1}}^2\right) &= \sum_{t=1}^T \log\left\vert \mbold{V}_{t+1}\right\vert - \log\left\vert \mbold{V}_{t}\right\vert &\\
     &= \log\left(\frac{\det(\mbold{V}_{T+1})}{\det(\lambda\mbold{I}_d)}\right) &\text{(telescopic sum)} \\
     &= \log\left(\frac{\det(\mbold{V}_{T+1})}{\lambda^d}\right)
\end{align*}
Remember that for all $x\in[0,1]$ we have the inequality $\log(1+x)\geq x/2$. Also note that $\left\lVert x_t\right\rVert_{\mbold{V}_t^{-1}}^2\leq X^2/\lambda$. Therefore:
\begin{align*}
     \log\left(\frac{\det(\mbold{V}_{T+1})}{\lambda^d}\right) &=  \sum_{t=1}^T \log\left(1+\left\lVert x_t\right\rVert_{\mbold{V}_t^{-1}}^2\right)&\\
     &\geq \sum_{t=1}^T \log\left(1+\frac{1}{\max(1,X^2/\lambda)}\left\lVert x_t\right\rVert_{\mbold{V}_t^{-1}}^2\right)&\\
     &\geq \frac{1}{2\max(1,X^2/\lambda)}\sum_{t=1}^T \left\lVert x_t\right\rVert_{\mbold{V}_t^{-1}}^2
\end{align*}
which yields the announced result. 
\end{proof}

We will also need Lemma~10 of \cite{abbasi2011improved}. We remind its statement here for the sake of completeness.

\begin{lemma}[Determinant-Trace inequality]
     Let $\{x_s\}_{s=1}^\infty$ a sequence in $\mbb{R}^d$ such that $\ltwo{x_s}\leq X$ for all $s\in\mbb{N}$, and  let $\lambda$ be a non-negative scalar. For $t\geq 1$ define $\mbold{V}_t \defeq \sum_{s=1}^{t-1} x_sx_s^T+\lambda\mbold{I}_d$. The following inequality holds:
     \begin{align*}
         \det(\mbold{V}_{t+1}) \leq \left(\lambda+tX^2/d\right)^d
     \end{align*}
\label{lemma:determinant_trace_inequality}
\end{lemma}

\end{document}